\documentclass{article}

\usepackage{microtype}
\usepackage{graphicx}
\usepackage{subfigure}
\usepackage{booktabs} %
\usepackage{placeins}

\usepackage{hyperref}

 \usepackage[accepted]{icml2023}

\usepackage{amsmath}
\usepackage{amsfonts}
\usepackage{amsthm}

\theoremstyle{plain}

\newtheorem{theorem}{Theorem}[section]
\newtheorem{corollary}[theorem]{Corollary}
\newtheorem{lemma}[theorem]{Lemma}

\newtheorem{definition}[theorem]{Definition}

\newcommand{\onev}{\mathbf{1}}
\newcommand{\zerov}{\mathbf{0}}

\newcommand{\tO}[1]{\widetilde{O}\left(#1\right)}

\newcommand{\tx}{\widetilde{x}}

\newcommand{\eps}{\varepsilon}

\newcommand{\ot}{\bar{t}}

\newcommand{\bb}{\boldsymbol{\mathit{b}}}

\newcommand{\ggg}{\boldsymbol{\mathit{g}}}

\newcommand{\ww}{\boldsymbol{\mathit{w}}}

\newcommand{\xx}{\boldsymbol{\mathit{x}}}
\newcommand{\oxx}{\boldsymbol{\bar{\mathit{x}}}}
\newcommand{\txx}{\boldsymbol{\widetilde{\mathit{x}}}}
\newcommand{\yy}{\boldsymbol{\mathit{y}}}

\renewcommand{\AA}{\boldsymbol{\mathit{A}}}

\newcommand{\II}{\boldsymbol{\mathit{I}}}

\newcommand{\OO}{\boldsymbol{\mathit{O}}}

\renewcommand{\epsilon}{\varepsilon}

\icmltitlerunning{Gradient Descent Converges Linearly for Logistic Regression on Separable Data}

\newcommand{\Procedure}{\FUNCTION}
\newcommand{\EndProcedure}{\ENDFUNCTION}
\newcommand{\State}{\STATE}
\newcommand{\For}{\FOR}
\newcommand{\EndFor}{\ENDFOR}

\newcommand{\Return}{\STATE {\bf return} }

\begin{document}

\twocolumn[
\icmltitle{Gradient Descent Converges Linearly \\ for Logistic Regression on Separable Data}

\begin{icmlauthorlist}
\icmlauthor{Kyriakos Axiotis}{a}
\icmlauthor{Maxim Sviridenko}{b}
\end{icmlauthorlist}

\icmlaffiliation{a}{Google Research, New York, NY, USA}
\icmlaffiliation{b}{Yahoo! Research, New York, NY, USA}

\icmlcorrespondingauthor{Kyriakos Axiotis}{axiotis@google.com}
\icmlcorrespondingauthor{Maxim Sviridenko}{sviri@yahooinc.com}

\icmlkeywords{Machine Learning, ICML}

\vskip 0.3in
]

\printAffiliationsAndNotice{This work was performed while the first author was at MIT.}  %

\begin{abstract}
We show that running gradient descent with variable learning rate guarantees loss 
$f(\xx) \leq 1.1 \cdot f(\xx^*) + \epsilon$
for the logistic regression objective, 
where the error $\epsilon$ decays exponentially with the number of iterations and polynomially with the magnitude of the entries of 
an arbitrary fixed solution $\xx^*$. 
This is in contrast to the common intuition that the absence of strong convexity precludes linear 
convergence of first-order methods, and highlights the importance of variable learning rates
for gradient descent. 
We also apply our ideas to sparse logistic regression, where they lead to an exponential improvement of the sparsity-error tradeoff.
\end{abstract}

\section{Introduction}

Logistic regression is one of the most widely used classification methods because of its simplicity, interpretability, and good practical performance. Yet, the convergence behavior of first-order methods
on this task is not well understood: In practice gradient descent performs much better than what the theory predicts.
In particular, a general analysis of gradient descent for smooth functions implies convergence with the error in function value decaying as $O(1/T)$.
Analyses with stronger, linear convergence guarantees generally require the function to satisfy the strong convexity property, which, in contrast
to other losses such as the $\ell_2$ loss, the logistic loss only satisfies in a bounded set of solutions around zero.
As a result, this introduces an \emph{exponential} runtime dependency on the magnitude of the 
target solution~\cite{ratsch2001convergence,freund2018condition}, which is undesirable in practice.
This poses a serious obstacle to obtaining high-precision solutions for logistic regression.

In fact, it was shown in~\cite{telgarsky2012primal}
that the $\mathrm{poly}(1/T)$ bound on function value convergence is tight for gradient descent on
general (non-linearly separable) data.
The significance of the separability of the data for convergence has also been observed 
in~\cite{telgarsky2013margins,ji2018risk,freund2018condition}, who present
convergence results based on quantitative measures of separability. 

A deeper study into the structure of both the exponential and logistic losses 
for separable data was initiated
by~\cite{telgarsky2012primal}, who showed that greedy coordinate descent
achieves linear convergence with a rate that depends on the maximum linear classification margin 
(i.e. hard SVM margin). Unfortunately, for logistic regression, it also has a $2^m$ dependence on the 
number of examples, making it inefficient for real-world tasks. 
\cite{telgarsky2013margins} refines the results of~\cite{telgarsky2012primal} for the exponential loss, 
but for logistic regression still suffers from an exponential overhead originating 
from the multiplicative discrepancy between the exponential and logistic losses.
Interestingly, however, the authors note (\cite{telgarsky2013margins}, Section 5)
that logistic regression 
experiments paint a much more favorable picture than the theory predicts.

A related line of work deals with convergence to the maximum-margin classifier on linearly
separable classification instances using gradient descent.
\cite{soudry2018implicit,ji2018risk} showed that the estimator obtained by optimizing the logistic
or the exponential loss with gradient descent
converges to the maximum-margin linear classifier at a rate of 
$O(\log \log T / \log T)$ (in $\ell_2$ norm). 
For the exponential loss, \cite{nacson2019convergence} showed that
the convergence bound to the maximum margin estimator can be exponentially improved to
$O(\log T / \sqrt{T})$, by using gradient descent with variable (increasing) learning rate.
The authors' experiments indicate that variable step sizes could lead to a similar exponential
improvements for the case of logistic regression and shallow neural networks.
Recently, \cite{ji2021characterizing} presented a novel primal-dual approach that proves
that the latter claim indeed holds for the logistic regression and exponential objectives,
obtaining a maximum-margin error decaying as $O(1/T)$, using a variable learning rate. This exponentially
improved upon the results of~\cite{soudry2018implicit,ji2018risk}.

Another approach to obtain high-precision solutions is by using second order methods, which in addition to first order (gradient) information, use second order (Hessian) information about the function. 
These make use of second order stability properties, such as quasi-self-concordance~\cite{bach2010self} combined with Newton's method~\cite{karimireddy2018global}, 
or ball oracles~\cite{carmon2020acceleration,adil2021unifying}.
Such approaches are generally not suitable for large-scale applications because of their reliance on repeated calls to large linear system solvers.

\paragraph{Our work.}
In this paper, we show that (under appropriate assumptions) we can get the best of both worlds of first and second order methods, thus giving a partial explanation for the excellent performance
that first-order methods have for logistic regression in practice. In particular, given a binary classification instance 
$(\AA\in\{-1,1\}^{m\times n}, \bb\in\{-1,1\}^{m})$ with associated logistic
loss $f(\xx) = \sum\limits_{i} \log (1 + \exp(-b_i (\AA\xx)_i))$,
we show that simple variants of gradient descent return a solution with $f(\xx) \leq (1+\delta) \cdot f(\xx^*) + \eps$ after $O\left( K\left(\frac{1}{\delta} + \log \frac{f(\zerov)}{\eps}\right)\right)$
iterations, where $K = \mathrm{poly}(n, \left\|\xx^*\right\|)$ and $x^*$ is an arbitrary fixed solution.
Even though the error still decays as $1/T$ in the worst case because of the $\frac{1}{\delta}$ dependence, the additive error is now
$\delta f(\xx^*)$ instead of $\delta f(\zerov)$, allowing for much faster convergence when the optimal loss $f(\xx^*)$ is smaller (which is our measure of linear separability of the data). For linearly
separable data, i.e. as $f(\xx^*)$ approaches $0$, the convergence becomes linear. 

Instead of properties like Lipschitzness, smoothness, strong convexity that are commonly used in the study of first order methods, we find that there are two properties that are more relevant to the 
structure of the logistic regression problem. The first one is \emph{second order robustness}, which means that the Hessian is stable (in a spectral sense) in any small enough norm ball~\cite{cohen2017matrix}.
This is closely related to quasi-self-concordance, a property that has been previously used in the analysis of second order algorithms~\cite{bach2010self}. The second property is what we call 
\emph{multiplicative smoothness}, which means that the function is locally smooth, with the smoothness constant being proportional to the function value (loss). 
A similar property was used by~\cite{ji2018risk} to prove convergence of logistic regression to the
maximum margin classifier.
Together, these properties 
show that, as the loss decreases, the objective becomes (locally) smoother and therefore the 
learning rate can increase. This motivates a
variable step size schedule that is inversely proportional to the loss, thus making larger steps as the solution approaches optimality. This in fact agrees with the observations 
of~\cite{soudry2018implicit,nacson2019convergence} on the importance of a variable learning rate. As can be seen in the toy example from~\cite{soudry2018implicit} in Figure~\ref{fig:toy}, 
simply replacing the fixed learning rate $\eta$ used in~\cite{soudry2018implicit}
by an increasing learning rate $\eta \cdot f(\xx^0) / f(\xx^T)$ yields an exponential improvement, both in loss and distance to the maximum margin estimator.

\begin{figure}[h]
\begin{center}
\includegraphics[width=0.49\columnwidth]{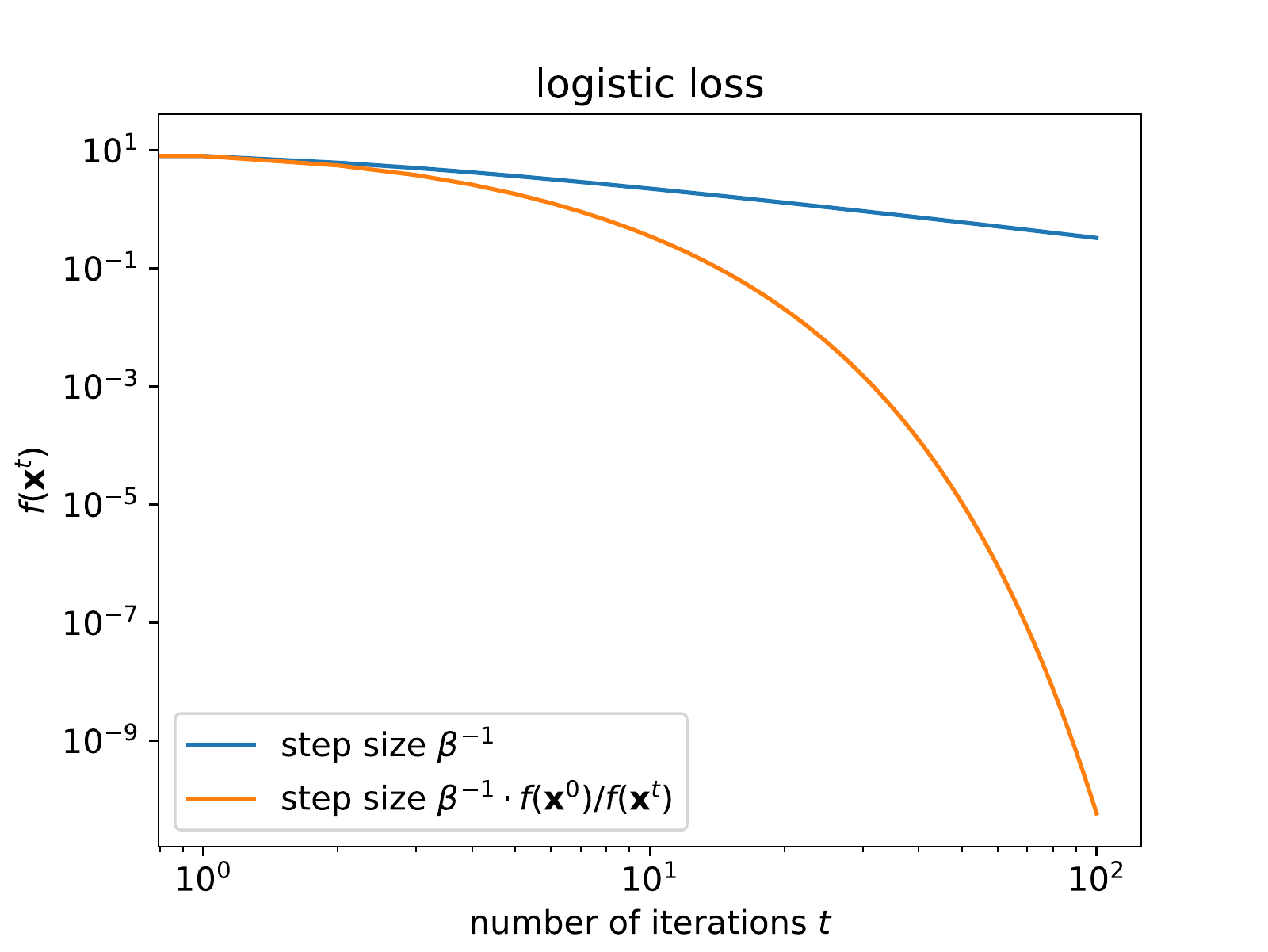}
\includegraphics[width=0.49\columnwidth]{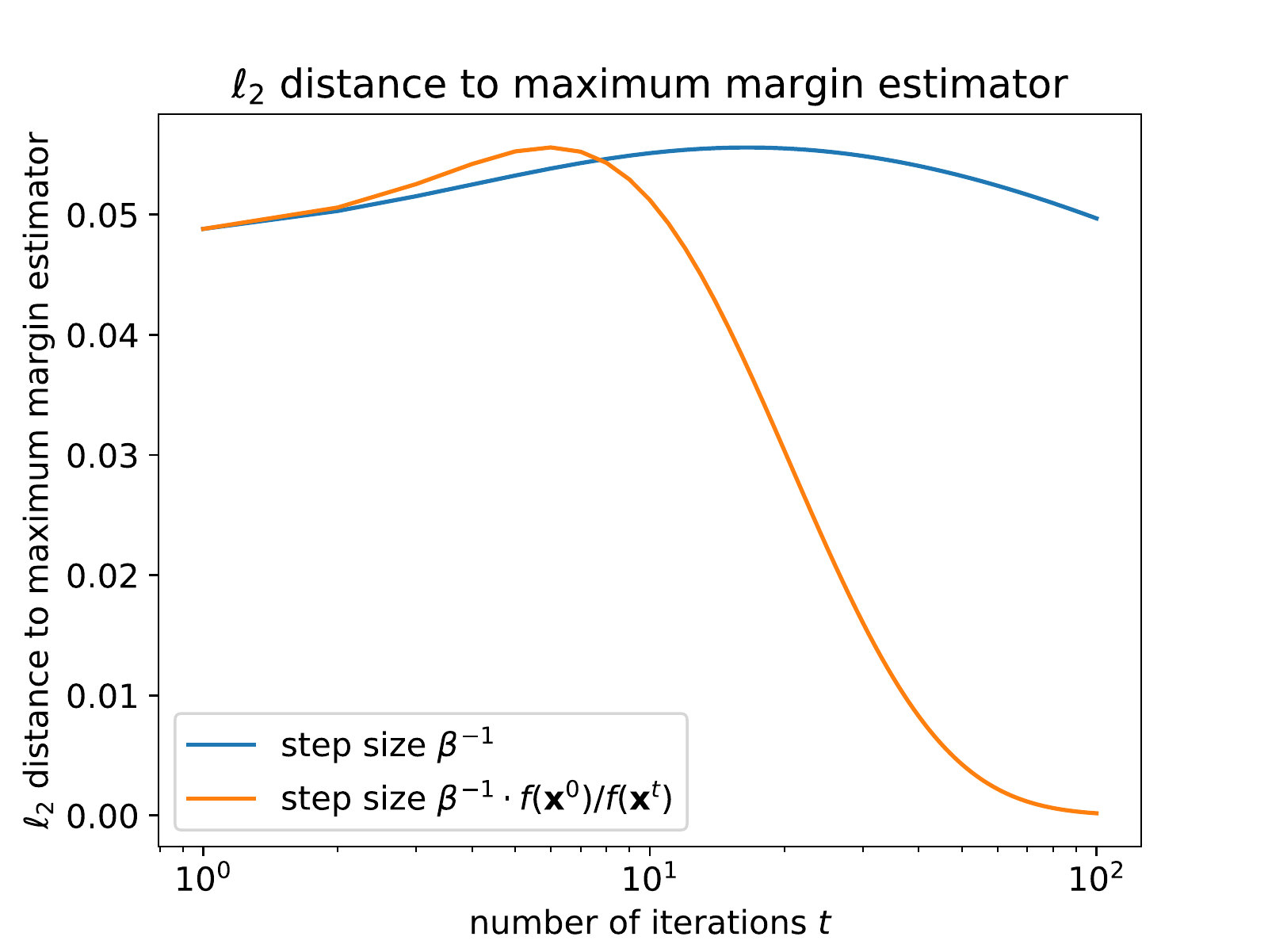}
\end{center}
\caption{Comparison between fixed and increasing step sizes in the toy example from Figure 1 of~\cite{soudry2018implicit}. The fixed step size is set to $\beta^{-1} := \left\|\AA\right\|_2^{-2}$, and the increasing
to $\beta^{-1} f(\xx^0) / f(\xx^T)$. The estimator error is defined as $\left\|\xx^t/\left\|\xx^t\right\|_2 - \xx^* / \left\|\xx^*\right\|_2\right\|_2$.}
\label{fig:toy}
\end{figure}

\bgroup
\def\arraystretch{1.5}%
\begin{table*}[t]
\caption{Algorithms for logistic regression and dependence on $m/\eps$ (omitting extra $\mathrm{polylog}(m,n)$ factors). Algorithms with exponential dependences on any problem parameter are omitted.
For example, the standard gradient descent analysis shows linear convergence, but with an exponential dependence on $\left\|\xx^*\right\|_\infty$ (see e.g. \citet{freund2018condition}).}
\label{table:1}
\vskip 0.15in
\centering
\begin{small}
\begin{sc}
\begin{tabular}{lccr}
 Algorithm & Order & Guarantee  & Runtime Error Dependence\\
 \hline
 \hline
 Gradient descent & First & $f(\xx) \leq f(\xx^*) + \eps$  & $ m / \eps$ \\
 Accelerated gradient descent & First & $f(\xx) \leq f(\xx^*) + \eps$ & 
 $\sqrt{m/\eps}$ \\ 
 Newton/Trust region %
 & Second
 & $f(\xx) \leq f(\xx^*) + \eps$ & $\log(m/\eps)$ \\
 This paper & First & $f(\xx) \leq (1+\delta) \cdot f(\xx^*) + \eps$ & $\delta^{-1} + \log (m/\eps)$
\end{tabular}
\end{sc}
\end{small}
\vskip -0.1in
\end{table*}
\egroup

\subsection{Sparse logistic regression}

\bgroup
\def\arraystretch{1.5}%
\begin{table*}[t]
\caption{Algorithms for sparse logistic regression and asymptotic sparsity dependences.}
\label{table:2}
\vskip 0.15in
\centering
\begin{small}
\begin{sc}
\begin{tabular}{lccr}
 Algorithm & Guarantee  & Sparsity & Order\\
 \hline
 \hline
 \cite{SSZ10} & $f(\xx) \leq f(\xx^*) + \eps$  & $\left\|\xx^*\right\|_1^2 m/\eps$ & First\\
 This paper & $f(\xx) \leq (1+\delta) \cdot f(\xx^*) + \eps$  & $\left\|\xx^*\right\|_1^2 (\delta^{-1} + \log(m/\eps))$ & First \\
\end{tabular}
\end{sc}
\end{small}
\vskip -0.1in
\end{table*}
\egroup
In practice, it is often important to force the solution of a logistic regression problem to be
\emph{sparse}, i.e. have only a few non-zero entries, which is a form of feature selection.
This is because most of the features might only be marginally useful, 
and thus one can drastically reduce the size of the model while
not significantly sacrificing the predictive performance. Apart from computational efficiency, 
feature selection is also important to improve interpretability and avoid overfitting.

Most progress in sparse optimization has focused on objective functions with condition number bounded by 
some $\kappa > 0$. Results in this line of work guarantee a solution with relaxed sparsity
$s' \geq s$, where $s$ is the target sparsity, and algorithms include lasso, 
orthogonal matching pursuit (OMP),
and 
iterative hard 
thresholding (IHT)~\cite{Natarajan95,IHT,SSZ10,JTD11,JTK14,axiotis2021sparse,axiotis2022iterative}. 
The state of the art 
result by~\cite{axiotis2022iterative} gives a sparsity of $s'=O(\kappa) \cdot s$ using
a variant of the IHT algorithm.

However, the condition number of the logistic loss is unbounded, because it is not strongly convex.
Therefore, these results do not directly apply,
although they do apply to $\ell_2$-regularized logistic regression. 
Some works~\cite{van2008high,bunea2008honest}
have analyzed lasso methods for logistic regression without condition number
assumptions, and
\cite{SSZ10} provides three
different analyses for smooth but not strongly convex functions. These apply to logistic regression and 
give a sparsity of $O\left(\left\|\xx^*\right\|_1^2 \frac{m}{\eps}\right)$
to achieve a loss of $f(\xx) \leq f(\xx^*) + \eps$. The most practical of these is
a forward greedy selection algorithm, which is also known as greedy coordinate descent.

{\bf Our work.} Using the second order stability and multiplicative smoothness properties, we show that
a slight variation of greedy coordinate descent gives a sparsity of 
\[ O\left(\left\|\xx^*\right\|_1^2(\delta^{-1} + \log (m/\eps))\right) \]
and a loss of $f(\xx) \leq (1+\delta) \cdot f(\xx^*) + \eps$. As long as the $1+\delta$ approximation in 
front of $f(\xx^*)$ is tolerated, as is the case when $f(\xx^*) \ll m$, this implies an exponential
improvement in the $\eps$ dependence from $\frac{m}{\eps}$ to $\log \frac{m}{\eps}$.
In addition, our analysis is compatible with incorporating fully corrective steps to the algorithm. These are steps that are occasionally performed to optimize over the support of the current solution (i.e. fully optimize the weights of the currently selected features) and is often used in applications like feature selection.
\section{Preliminaries}
\label{appendix:prelim}
\paragraph{Notation.} We denote $[n] = \{1,2,\dots,n\}$. We will use {\bf bold} to refer to vectors or matrices. We denote by $\zerov$ the all-zero vector, $\onev$ the all-one vector,
$\OO$ the all-zero matrix, and by $\II$ the identity matrix (with dimensions understood from the context). Additionally, we will denote by $\onev_i$ the $i$-th basis vector, i.e.
the vector that is $0$ everywhere except at position $i$.

In order to ease notation and where not ambiguous for two vectors $\xx, \yy\in \mathbb{R}^n$, we denote by $\xx \yy\in \mathbb{R}^n$ a vector with elements $(\xx\yy)_i=x_iy_i$, i.e.
the element-wise multiplication of two vectors $\xx$ and $\yy$. In contrast,
we denote their inner product by $\langle \xx,\yy\rangle$ or $\xx^\top \yy$. Similarly, $\xx^2\in \mathbb{R}^n$ will be the element-wise square of vector $\xx$.
For any function $g(t)$ of a single variable, let $g(\xx)\in \mathbb{R}^n$
  be a vector with elements $(g(\xx))_i=g(x_i)$, e.g. will use this notation when  $g$ is the sigmoid function.

For any vector $\xx\in\mathbb{R}^n$ and set $S\subseteq[n]$, we denote by
$\xx_S$ the vector that results from $\xx$ after zeroing out all the entries except those in positions given by indices in $S$. We will also use the notation $\nabla_S f(\xx) := (\nabla f(\xx))_S$ to denote
the restriction of a gradient to $S$. We also denote by 
$\mathrm{supp}(\xx) := \{i\in[n]\ |\ x_i \neq 0\}$ the \emph{support} of $\xx$.

We use the notation $\tO{\cdot}$ to hide $\mathrm{poly}\log(n,m)$ factors in $O$-notation, where $n$ is the dimension of the problem and $m$ is the number of examples.

\paragraph{Norms.}
For any $p\in(0,\infty)$ and weight vector $\ww\geq \zerov$, we define the weighted $\ell_p$ norm of a vector $\xx\in\mathbb{R}^n$ as:
\begin{align*}
\left\|\xx\right\|_{p,\ww} = \left(\sum\limits_{i} w_i x_i^p \right)^{1/p}\,.
\end{align*}
For $p=0$, we denote $\left\|\xx\right\|_0 = \left|\{i\ |\ x_i\neq 0\}\right|$ to be the \emph{sparsity} of $\xx$.
For $p=\infty$, we denote $\left\|\xx\right\|_\infty = \max_i |x_i|$ to be the maximum absolute value of $\xx$.

For a matrix $\AA\in\mathbb{R}^{m\times n}$, we let $\left\|\AA\right\|_{p\rightarrow q}$ 
be its $p$ to $q$ operator norm, defined as
\begin{align*}
\left\|\AA\right\|_{p\rightarrow q} = \underset{\xx\neq \zerov}{\max}\, 
\frac{\left\|\AA \xx\right\|_q}{\left\|\xx\right\|_p}\,.
\end{align*}
In particular, $\left\|\AA\right\|_{1\rightarrow \infty}$ is equal to the largest entry of $\AA$ in 
absolute value.

\paragraph{Smoothness and convexity.}
A differentiable function $f:\mathbb{R}^n\rightarrow\mathbb{R}$ is called \emph{convex} if for any
$\xx,\yy\in\mathbb{R}^n$ we have $f(\yy) \geq f(\xx) + \langle \nabla f(\xx), \yy-\xx\rangle$.
Furthermore, 
$f$ is called \emph{$L$-Lipschitz (with respect to some norm $\left\|\cdot\right\|$)}
for some real number $L > 0$ if for any $\xx, \yy\in \mathbb{R}^n$ we have
$\left|f(\yy) - f(\xx)\right| \leq L \left\|\yy-\xx\right\|$, and 
\emph{$\beta$-smooth} if for any $\xx,\yy\in\mathbb{R}^n$ we have
$\left\|\nabla f(\yy) - f(\xx)\right\|_* \leq \beta \left\|\yy-\xx\right\|^2$, where
$\left\|\cdot\right\|_*$ is the \emph{dual norm} of $\left\|\cdot\right\|$.
If $f$ is only $\beta$-smooth
along $s$-sparse directions (i.e. only for $\xx,\yy\in\mathbb{R}^n$ such that 
$\left\|\yy-\xx\right\|_0 \leq s$), then we call $f$ $\beta$-smooth \emph{at sparsity level $s$}
and denote the smallest such $\beta$ by $\beta_s$ and call it the 
\emph{restricted smoothness constant} (at sparsity level $s$).
\section{Logistic Regression Analysis via Multiplicative Smoothness}
\label{sec:mul}

In the logistic regression problem, our goal is to
minimize the function $f(\xx) = \sum\limits_{i=1}^m \log(1+e^{-(\AA\xx)_i})$, where 
$\AA\in\mathbb{R}^{m\times n}$ is a data matrix\footnote{This formulation is without loss of generality, because we can incorporate the binary $\pm1$ labels into the matrix $\AA$ and assume that all the labels are
positive.}.

Our starting point, as is usually the case with first-order methods, 
will be the second order Taylor expansion of $f$:
\begin{align}
f(\xx + \txx) = f(\xx) + \langle \nabla f(\xx), \txx\rangle + \frac{1}{2} \langle \txx, \nabla^2 f(\oxx) \txx\rangle\,,
\label{eq:taylor}
\end{align}
where, by the mean value theorem for twice continuously
differentiable functions, $\oxx$ is entry-wise between $\xx$
and $\xx'=\xx + \txx$, and $\nabla^2 f(\oxx)$ is the Hessian of $f$ at $\oxx$. 

\paragraph{Second-order robustness.} In fact, as long as the step $\txx$ is not too large,
the Hessian at $\oxx$ will not differ much (spectrally) from the Hessian at $\xx$.
This is because of the following property of the logistic function called
\emph{second order robustness}~\cite{cohen2017matrix}, which is also 
very closely related to quasi-self-concordance~\cite{bach2010self}.
\begin{definition}[Second-order robustness]
A twice differentiable function $f:\mathbb{R}^n\rightarrow \mathbb{R}$ is called 
$q$-\emph{second order robust} with respect to a norm $\left\|\cdot\right\|$ if 
its Hessian is stable in any $(1/q)$-sized $\left\|\cdot\right\|$-ball, i.e.
for any $\xx,\xx'\in\mathbb{R}^n$ such that 
$\left\|\xx'-\xx\right\| \leq 1/q$, 
we have $\frac{1}{2} \nabla^2 f(\xx) \preceq \nabla^2 f(\xx') \preceq 2 \nabla^2 f(\xx)$.
\end{definition}
It is not hard to see that $f$ is $2M$-second order robust with respect to the $\ell_1$ norm, 
where $M$ is an upper bound on the entries of $\AA$ in absolute value.
\begin{lemma}[Second-order robustness of the logistic loss]
The function $f(\xx) = \sum\limits_{i=1}^m \log(1+e^{-(\AA\xx)_i})$ is $2M$-second order robust
with respect to the $\ell_1$ norm, where $M$ is the largest entry of $\AA$ in absolute value.
\label{lem:second_order_robustness}
\end{lemma}
\begin{proof}
For any $\xx\in\mathbb{R}^n$, we have $\nabla^2 f(\xx) = \AA^\top \mathrm{diag}(w(\AA\xx)) \AA$,
where $w(\AA\xx) = \sigma(\AA\xx)(1-\sigma(\AA\xx))$ and $\sigma(t) = 1/(1+e^{-t})$ is the sigmoid
function, that we extend to vectors by applying it elementwise.
We define $r(t) := \log w(t)$, which is a $1$-Lipschitz function, because
\begin{align*}
r'(t) = \frac{w'(t)}{w(t)} = \frac{w(t)\cdot (1-2\sigma(t))}{w(t)} = 1-2\sigma(t)\,,
\end{align*}
whose absolute value is always bounded by $1$. 
Therefore, for any $|t' - t| \leq 1/2$, we have 
\begin{align*}
& |r(t') - r(t)| \leq 1/2 \\
 \Leftrightarrow  &
 \left|\log \frac{w(t')}{w(t)}\right| \leq 1/2\\
 \Rightarrow  &
 \frac{1}{2} w(t)\leq w(t') \leq 2 w(t)\,,
\end{align*}
and consequently for any $\xx,\xx'$ such that 
\[ 
\left\|\xx'-\xx\right\|_1 \leq \frac{1}{2M} 
\Rightarrow 
\left\|\AA\xx'-\AA\xx\right\|_\infty \leq 1/2 
\,,\]
we have that $\frac{1}{2} w(\AA\xx) \leq w(\AA\xx') \leq 2 w(\AA\xx)$.
This immediately implies that $\frac{1}{2} \nabla^2 f(\xx) \preceq \nabla^2 f(\xx') \preceq 2 \nabla^2 f(\xx)$,
concluding the proof.
\end{proof}
Because of Lemma~\ref{lem:second_order_robustness}, (\ref{eq:taylor}) implies the much simpler
\begin{align}
f(\xx + \txx) \leq f(\xx) + \langle \nabla f(\xx), \txx\rangle + \langle \txx, \nabla^2 f(\xx) \txx\rangle\,,
\label{eq:taylor2}
\end{align}
as long as $\left\|\txx\right\|_1 \leq 1/(2M)$.

\paragraph{Multiplicative smoothness.}
We can easily calculate that
$\nabla f(\xx) = -\AA^\top \left(\onev - \sigma(\AA\xx)\right)$, where $\sigma(t) = 1/(1+e^{-t})$ is the
sigmoid function, and 
$\nabla^2 f(\xx) = \AA^\top \mathrm{diag}(w(\AA\xx)) \AA$, 
where $w(\AA\xx) = \sigma(\AA\xx) (\onev - \sigma(\AA\xx))$ are diagonal weights.
Now, we should note that the second order term of (\ref{eq:taylor2}) can be re-written as
$\langle w(\AA\xx), (\AA\txx)^2\rangle$. This term,
whose magnitude is what will determine the step size of the algorithm and in turn the bound on
the total number of iterations, becomes smaller as the weights $w(\AA\xx)$ become smaller.
The crucial observation is that these weights are bounded in a way that depends on the 
logistic loss of the solution $\xx$, as shown in the lemma below:
\begin{lemma}[Sum of second derivatives of the logistic function]
Let $f(\xx) = \sum\limits_{i=1}^m \log(1+e^{-(\AA\xx)_i})$ and
$w(t) = \sigma(t)(1-\sigma(t))$, where $\sigma$ is the sigmoid function. Then,
\begin{align}
\sum\limits_{i=1}^m w((\AA\xx)_i) \leq f(\xx)\,.
\end{align}
\label{lem:sum_2nd_derivative}
\end{lemma}
\begin{proof}
We have
$w(t) = \sigma(t) (1 - \sigma(t)) \leq 1-\sigma(t) \leq -\log \sigma(t)
=\log(1+e^{-t})$,
where we used the inequality $\log p \leq p-1$ for all $p > 0$.
\end{proof}
In other words, as the loss decreases, $f$ becomes \emph{smoother} (in an appropriate sense).
This is what allows the algorithm to employ a step size that is \emph{inversely proportional} to the loss.
A similar observation has been made in~\cite{ji2018risk}.
The above discussion motivates the following
definition of \emph{multiplicative smoothness}. This is related to the usual 
definition of smoothness but also
incorporates the property that the function becomes smoother as the loss decreases.
\begin{definition}[Multiplicative smoothness]
We call a differentiable function $f:\mathbb{R}^n\rightarrow \mathbb{R}_{>0}$
$\mu$-\emph{multiplicatively smooth} with respect to a norm $\left\|\cdot\right\|$, if for any 
$\xx,\xx'\in\mathbb{R}^n$ we have
\begin{align*}
\frac{\txx^\top \nabla^2 f(\xx) \txx}{f(\xx)} \leq \mu\left\|\txx\right\|^2\,.
\end{align*}
\end{definition}
Our use of a general norm is not an over-generalization, since as we will see the 
$\ell_1$ norm is more suitable for sparse logistic regression, and the $\ell_2$ norm
is more suitable for the unrestricted case. 
In fact, it can be proved that the logistic loss
is $M^2$-multiplicatively smooth with respect to the $\ell_1$ norm, where we remind that $M$ is a bound on the entries of $\AA$ in absolute value:
\begin{lemma}[Multiplicative smoothness of the logistic loss]
The function $f(\xx) = \sum\limits_{i=1}^m \log(1+e^{-(\AA\xx)_i})$ is $M^2$-multiplicatively smooth
with respect to the $\ell_1$ norm, where $M$ is the largest entry of $\AA$ in absolute value.
\label{lem:multiplicative_smoothness}
\end{lemma}
\begin{proof}
For any $\xx,\txx\in\mathbb{R}^n$, using the fact that 
$\nabla^2 f(\xx) = \AA^\top \mathrm{diag}\left(w(\AA\xx)\right) \AA$, where
$w$ is as defined in Lemma~\ref{lem:sum_2nd_derivative}, we have
\begin{align*}
\txx^\top \nabla^2 f(\xx) \txx
& = \sum\limits_{i=1}^m w(\AA\xx)_i (\AA\txx)_i^2\\
& \leq \sum\limits_{i=1}^m w(\AA\xx)_i \left\|\AA\txx\right\|_\infty^2\\
& \leq M^2 \sum\limits_{i=1}^m w(\AA\xx)_i \left\|\txx\right\|_1^2\\
& \leq M^2 f(\xx) \left\|\txx\right\|_1^2\,,
\end{align*}
where the second to 
last inequality follows from the fact that the entries of $\AA$ are bounded by $M$,
and the last inequality from Lemma~\ref{lem:sum_2nd_derivative}.
\end{proof}

In the following sections, we will see how the second order robustness and multiplicative smoothness properties play into the design and analysis of algorithms for sparse and general
logistic regression.

\section{Sparse logistic regression}
\label{sec:sparse}

As we saw, the logistic loss is $2M$-second order robust and $M^2$-multiplicatively smooth with respect to the $\ell_1$ norm. This is an ideal norm for \emph{sparse} logistic regression,
where in addition to minimizing the loss we want to restrict the solution to have few non-zero entries. In 
particular, it yields a variant of the $\ell_1$ gradient descent algorithm (aka
greedy coordinate descent), which is presented in Algorithm~\ref{alg:sparse}.

\begin{algorithm}
    \caption{Greedy Coordinate Descent}
    \begin{algorithmic}[1] %
        \Procedure{GreedyCoordinateDescent($\xx^0, T, M, B$)} 
			\State Let $f(\xx) := \sum\limits_{i=1}^m \log(1 + e^{-b_i (\AA\xx)_i})$
			\For{$t=0,\dots,T-1$}
				\State For all $i\in[n]$ define $\zeta_i = \begin{cases}
				\lambda_t & \text{ if $x_i^t = 0$}\\
				0 & \text{ if $\left|x_i^t\right| \geq B$ and $\nabla_i f(\xx^t) \cdot x_i^t < 0$}\\
				1 & \text{ otherwise}\\
				\end{cases}$
				\State $i' \leftarrow \mathrm{argmax}_i \left\{\zeta_i \left|\nabla_i f(\xx^t)\right|\right\}$
				\State $\eta \leftarrow \left(2M^2 f(\xx^t)\right)^{-1}$
				\State $x_{i'}^{t+1} \leftarrow x_{i'}^t - \eta \nabla_{i'} f(\xx^t)$
			\EndFor
			\Return $\xx^T$
        \EndProcedure
    \end{algorithmic}
	\label{alg:sparse}
\end{algorithm}

The first thing that should be noted about this algorithm is the crucial parameters $\lambda_t$. These
parameters offer a quantitative threshold between sparsity and speed of convergence. In particular,
when $\lambda_t$ is $1$, then all entries (regardless of whether they are zero or not) are treated the 
same.
When $\lambda_t \ll 1$, on the other hand, the gradient entries corresponding to
zero entries are discounted by a factor $\ll 1$,
thus making the algorithm less eager to update these as opposed to non-zero entries, whose update
doesn't increase sparsity.

We are ready for the main theorem of this section.
In the proof, which can be found in Appendix~\ref{sec:proof_thm_sparse},
we present an analysis of Algorithm~\ref{alg:sparse} for sparse logistic regression.
In addition to an upper bound $B\geq \left\|\xx^*\right\|_\infty$, it also requires an 
approximation $B_1$ of $\left\|\xx^*\right\|_1$. 
One possible approach is to approximate it by $B$, but in practice this would be a learning rate
hyperparameter to be tuned.
We note $\xx^*$ is an arbitrary fixed solution (not necessarily optimal), whose entries are
bounded by $B$ in absolute value.

\begin{theorem}[Sparse logistic regression]
Given a binary classification instance $(\AA\in[-M,M]^{m\times n},\bb\in\{1,-1\}^m)$ 
and for any solution $\xx^*\in[-B,B]^n$ with 
$M \geq \max\{\left\|\xx^*\right\|_1^{-1}, B^{-1}\}$
\footnote{The theorem can be stated without this additional 
constraint,
but we include it because it makes the bounds considerably simpler.} and a known
parameter $B_1 \in \left[\frac{1}{C}\left\|\xx^*\right\|_1, \left\|\xx^*\right\|_1\right]$
for some $C \geq 1$,
Algorithm~\ref{alg:sparse} with 
$\lambda_t = \min\{B_1/\left\|\xx^t\right\|_1, 1\}$,
initial solution $\xx^0\in\mathbb{R}^n$,
and error tolerance $0 < \eps < m / 2$ returns a solution $\xx$ with
\begin{align*}
f(\xx) \leq (1+\delta) \cdot f(\xx^*) + \eps
\end{align*}
and sparsity 
\begin{align*} 
s' & := \left\|\xx\right\|_0 \\
& = 
O\left(\left\|\xx^*\right\|_1^2 M^2\left(\frac{1}{\delta} + 
\log \frac{f(\xx^0) - f(\xx^*)}{\eps}\right)\right)
\end{align*}
in
\begin{align*}
T = 
O\Big(
& \left(
\left\|\xx\right\|_0^2 
+ \left\|\xx^*\right\|_0^2 
\right)\\
& M^2 B^2 C^2
\left(\frac{1}{\delta} + \log \frac{f(\xx^0) - f(\xx^*)}{\eps}\right)\Big)
\end{align*}
iterations, for any choice of $\delta \in (0,1)$.
Each iteration consists of evaluating the logistic regression gradient $\nabla f$ plus $O(m+n)$ additional time.
\label{thm:sparse}
\end{theorem}
If the parameters $M,B,C$ are bounded by $\tO{1}$, we get a cleaner statement. 
\begin{corollary}
If $M,B,C\leq \tO{1}$ and $\xx^*$ is $s$-sparse, then Algorithm~\ref{alg:sparse} 
with $\lambda_t = \min\left\{B_1/\left\|\xx^t\right\|_1, 1\right\}$ 
returns a solution $\xx$ with
\begin{align*}
f(\xx) \leq 1.1 \cdot f(\xx^*) + \eps
\end{align*}
and sparsity 
\begin{align*}
s' & := \left\|\xx\right\|_0 \\
& = \tO{s^2 \log \frac{1}{\eps}}
\end{align*}
in
\begin{align*}
T = \tO{s^4 \log^3 \frac{1}{\eps}}
\end{align*}
iterations.
\label{cor:sparse}
\end{corollary}
Corollary~\ref{cor:sparse} follows from Theorem~\ref{thm:sparse} by applying the following 
series of inequalities:
\begin{align*}
\left\|\xx\right\|_0^2 
& \leq \tO{\left\|\xx^*\right\|_1^4\log^2\frac{1}{\eps}} \\
& \leq \tO{B^4\left\|\xx^*\right\|_0^4\log^2\frac{1}{\eps}}\\ 
& \leq \tO{\left\|\xx^*\right\|_0^4\log^2\frac{1}{\eps}}\,.
\end{align*}
It is useful to compare these results to the results of~\cite{SSZ10} for sparse optimization of
general smooth convex functions. Even though those results achieve the stronger error bound of 
$f(\xx) \leq f(\xx^*) + \eps$, the sparsity of the final solution is
in the order of $s^2 \frac{m}{\eps}$, which has an exponentially worse error dependence than
$s^2 \log \frac{m}{\eps}$. Therefore, if the approximation rate $(1+\delta)$ is tolerable
in front of $f(\xx^*)$, then one can obtain exponentially faster sparsity and convergence.

If we are willing to perform fully corrective steps as described in Algorithm~\ref{alg:sparse2}, 
then we can get a cleaner and slightly simpler analysis without dependence on parameters $B,B_1, C$. 
This is presented in Theorem~\ref{thm:sparse2} and proved in Appendix~\ref{sec:proof_thm_sparse2}.
Fully corrective steps can be useful when there is an efficient (dense) optimization algorithm and one wishes to use it as a black box for sparse optimization. In practice, one does not need to perform a full
correction, but only a small number of corrective gradient steps over the current support
of the solution.
\begin{algorithm}
    \caption{Greedy coordinate descent with fully corrective steps}
    \begin{algorithmic}[1] %
        \Procedure{FullyCorrectiveGreedyCD($\xx^0, T$)} 
			\State Let $f(\xx) := \sum\limits_{i=1}^m \log(1 + e^{-b_i (\AA\xx)_i})$
			\State $S^0 \leftarrow \mathrm{supp}(\xx^0)$
			\For{$t=0,\dots,T-1$}
				\State $i' \leftarrow \mathrm{argmax}_i \left\{\left|\nabla_i f(\xx^t)\right|\right\}$
				\State $S^{t+1} \leftarrow S^t \cup \{i'\}$
				\State $\xx^{t+1} \leftarrow \underset{\xx: \mathrm{supp}(\xx) \subseteq S^{t+1}}{\mathrm{argmin}}\, f(\xx)$
			\EndFor
			\Return $\xx^T$
        \EndProcedure
    \end{algorithmic}
	\label{alg:sparse2}
\end{algorithm}
\begin{theorem}[Sparse logistic regression with fully corrective steps]
Given a binary classification instance $(\AA\in[-M,M]^{m\times n},\bb\in\{1,-1\}^m)$ 
and for any solution $\xx^*\in\mathbb{R}^n$,
Algorithm~\ref{alg:sparse2} 
with error tolerance $0 < \eps < m / 2$ and initial solution $\xx^0$
returns a solution $\xx$ with
\begin{align*}
f(\xx) \leq (1+\delta) \cdot f(\xx^*) + \eps
\end{align*}
and sparsity 
\begin{align*}
s' & := \left\|\xx\right\|_0 \\
& =
\left\|\xx^0\right\|_0 + O\left(\left\|\xx^*\right\|_1^2 M^2 \left(\frac{1}{\delta} + 
\log \frac{f(\xx^0) - f(\xx^*)}{\eps}\right)\right)
\end{align*}
in
$T = O\left(\left\|\xx^*\right\|_1^2 M^2 \left(\frac{1}{\delta} + 
\log \frac{f(\xx^0) - f(\xx^*)}{\eps}\right)\right)$
iterations, 
for any choice of $\delta \in (0,1)$.
Each iteration consists of evaluating the logistic regression gradient $\nabla f$, solving a logistic regression problem on at most $s'$ variables,
plus $O(m+n)$ additional time.
\label{thm:sparse2}
\end{theorem}

\section{Dense logistic regression }
\label{sec:dense}

In this section, our goal is to minimize the logistic function $f$ without any constraint on
the sparsity of the solution. 
The results of Section~\ref{sec:sparse} applied to a 
full sparsity of $n$ already imply 
Corollary~\ref{cor:dense}.

\begin{corollary}[Dense logistic regression]
Given a binary classification instance $(\AA\in[-M,M]^{m\times n},\bb\in\{-1,1\}^m)$ 
and for any solution $\xx^*\in[-B,B]^n$
with $M \geq \max\left\{\left\|\xx^*\right\|_1^{-1}, B^{-1}\right\}$,
Algorithm~\ref{alg:sparse}
with $\lambda_t=1$ for all $t$, initial solution $\xx^0\in\mathbb{R}^n$,
and error tolerance $0 < \eps < m / 2$ returns a solution $\xx$ with
\begin{align*}
f (\xx) \leq (1+\delta) \cdot f(\xx^*) + \eps
\end{align*}
in
\begin{align*}
T = O\left(n^2M^2B^2
\left(\frac{1}{\delta} + \log \frac{f(\xx^0) - f(\xx^*)}{\eps}\right)\right)\,.
\end{align*}
iterations, for any choice of $\delta \in (0,1)$. Additionally, 
$\left\|\xx\right\|_\infty \leq B + \frac{1}{2M}$.
Each iteration consists of evaluating the logistic regression gradient $\nabla f$ plus $O(m+n)$ additional time.
\label{cor:dense}
\end{corollary}

Interestingly, even if we don't impose a sparsity constraint, the worst case analysis in Corollary~\ref{cor:dense} still obtains 
its bounds by using a greedy coordinate descent step, as in Section~\ref{sec:sparse}.
This is because it is a manifestation of $\ell_1$ gradient descent (aka steepest descent), which is rooted in the fact that the 
multiplicative smoothness of $f$ is with respect to the $\ell_1$ norm.
Greedy coordinate descent is favorable for sparse optimization, because it only updates one coordinate at a time. 
On the other hand,
based on practical intuitions, we would expect ($\ell_2$-based) gradient descent to perform better
than greedy coordinate descent if no sparsity constraint is imposed.
This is because greedy coordinate descent only updates one coordinate at a time, while gradient descent uses the full gradient.

In the rest of this section, we attempt to bridge this mismatch between worst case analysis and practice.
We show that, under appropriate assumptions, gradient descent can be proved to converge significantly faster than greedy coordinate descent.
We leave providing a theoretical grounding for these assumptions as an interesting open question for future research.
We now outline the core ideas.

\paragraph{$\ell_2$ multiplicative smoothness.} First, we can verify that the logistic loss does have the multiplicative
smoothness condition with respect to the $\ell_2$ norm, albeit in an almost trivial sense:
\begin{align*}
\langle \ww(\xx), (\AA\txx)^2\rangle 
& \leq \left\|\ww(\xx)\right\|_1 \left\|\AA\txx\right\|_\infty^2\\
& \leq f(\xx) \left\|\AA\right\|_{2\rightarrow\infty}^2 \left\|\txx\right\|_2^2\\
& \leq f(\xx) \beta \left\|\txx\right\|_2^2\,.
\end{align*}
Here, using the inequality $\left\|\AA\right\|_{2\rightarrow\infty}^2 \leq \left\|\AA\right\|_2^2 := \beta$
implies $\beta$-multiplicative smoothness with respect to the $\ell_2$ norm.
Unfortunately, this is not significantly better than the $\ell_1$ case: 
The number of iterations will be proportional to 
$\beta \left\|\xx^*\right\|_2^2$, which can be $\gg m$.
As it turns out, however, many
real logistic regression instances exhibit the $\ell_2$ multiplicative smoothness property
with significantly better constants. In our experiments we found that along the path 
of gradients encountered by gradient descent in a variety of instances, the following property was true:
\begin{align*}
\left\langle \ww(\xx), \left(\AA \nabla f(\xx)\right)^2\right\rangle
\leq f(\xx) \beta m^{-1} \left\|\nabla f(\xx)\right\|_2^2
\end{align*}
This is an \emph{effective} $\beta m^{-1}$-multiplicative smoothness property, because it is only assumed to be true
for $\xx$'s encountered by the gradient descent algorithm. As such, it is an empirical property.
In order to check our hypothesis, we have run the gradient descent algorithm with the step sizes
that are implied by Theorem~\ref{thm:dense_l2}, which we will see later.
For each of the 15 experiments, we have run gradient descent for 1000 iterations, and calculated the
maximum of the following quantity, over all iterations:
\begin{align*}
\frac{\left\langle \ww(\xx), \left(\AA \nabla f(\xx)\right)^2\right\rangle}
{f(\xx) m^{-1} \left\|\AA \nabla f(\xx)\right\|_2^2}\,.
\end{align*}
If this is bounded by $1$, and using the fact that 
$\left\|\AA \nabla f(\xx)\right\|_2^2 \leq \beta \left\|\nabla f(\xx)\right\|_2^2$, this implies that
$f$ is effectively $\beta m^{-1}$-multiplicatively smooth with respect to the $\ell_2$ norm. Indeed,
as we can see in Table~\ref{table}, these values are indeed less than $1$ for all datasets and all
iterations. 

\begin{table}[tb]
\caption{Upper bounds on the quantity 
$\left\langle \ww(\xx), \left(\AA \nabla f(\xx)\right)^2\right\rangle
 / \left(f(\xx) m^{-1} \left\|\AA \nabla f(\xx)\right\|_2^2\right)$. 
 Shown here is the maximum of this over $\xx$ being one of the first $1000$ iterates
 starting from $\xx^0=\zerov$.}
\label{table}
\begin{center}
\begin{tabular}{lr}
\multicolumn{1}{c}{Dataset}  &\multicolumn{1}{c}{Max ratio}
\\ \hline \\
letter & $0.40$\\
rcv1.test & $0.36$\\
ijcnn1 & $0.47$\\
vehv2binary & $0.37$\\
magic04 & $0.37$\\
skin & $0.44$\\
w8all & $0.40$\\
shuttle.binary & $0.37$\\
kddcup04.phy & $0.36$\\
kddcup04.bio & $0.48$\\
census & $0.50$\\
adult & $0.40$\\
poker & $0.36$\\
nomao & $0.50$\\
covtype & $0.36$
\end{tabular}
\end{center}
\end{table}

\ifx 0 
\begin{table*}[tb]
\caption{Upper bounds on the quantity 
$\left\langle \ww(\xx), \left(\AA \nabla f(\xx)\right)^2\right\rangle
 / \left(f(\xx) m^{-1} \left\|\AA \nabla f(\xx)\right\|_2^2\right)$. 
 Shown here is the maximum of this over $\xx$ being one one of the first $1000$ iterates.}
\label{table}
\begin{center}
\begin{tabular}{ll}
\multicolumn{1}{c}{Dataset}  &\multicolumn{1}{c}{Max ratio}
\\ \hline \\
letter & $0.40$\\
rcv1.test & $0.36$\\
ijcnn1 & $0.47$\\
vehv2binary & $0.37$\\
magic04 & $0.37$\\
\end{tabular}
\begin{tabular}{ll}
\multicolumn{1}{c}{Dataset}  &\multicolumn{1}{c}{Max ratio}
\\ \hline \\
skin & $0.44$\\
w8all & $0.40$\\
shuttle.binary & $0.37$\\
kddcup04.phy & $0.36$\\
kddcup04.bio & $0.48$
\end{tabular}
\begin{tabular}{ll}
\multicolumn{1}{c}{Dataset}  &\multicolumn{1}{c}{Max ratio}
\\ \hline \\
census & $0.50$\\
adult & $0.40$\\
poker & $0.36$\\
nomao & $0.50$\\
covtype & $0.36$
\end{tabular}
\end{center}
\end{table*}
\fi

In the following, our plan is to prove convergence, 
\emph{assuming} that $f$ has the multiplicative smoothness property with the constants 
in our hypothesis above. Under this assumption,
we can now prove a much stronger convergence theorem (here we are also using the fact that
$M^2 \leq \beta$ to replace $2M$- by $2\sqrt{\beta}$-second order robustness):
\begin{theorem}
Let $f:\mathbb{R}^n\rightarrow\mathbb{R}_{>0}$ be a convex function that is 
$\gamma$-second order robust and 
$\mu$-multiplicatively smooth 
with respect to the $\ell_2$ norm. 
Let $\xx^0\in\mathbb{R}^n$ be an initial solution and $\xx^*\in\mathbb{R}^n$ be an arbitrary solution,
where $R := \left\|\xx^0 - \xx^*\right\|_2$.
Then, gradient descent with step size 
$\eta_t = 
\min\left\{\frac{1}{2\mu f(\xx^t)}, \frac{1}{\gamma \left\|\nabla f(\xx^t)\right\|_2}\right\}$
returns a solution with
\begin{align*}
f(\xx) \leq (1+\delta) f(\xx^*) + \eps
\end{align*}
after
\begin{align*}
T = O\Big(& \mu R^2\left(\frac{1}{\delta} 
+ \log \frac{f(\xx^0) - f(\xx^*)}{\eps}\right)\\
& + \gamma R \log \frac{f(\xx^0) - f(\xx^*)}{\eps}
\Big)
\end{align*}
iterations.
If $\gamma \leq 2\sqrt{\beta}$ and $\mu \leq \beta m^{-1}$ for some $\beta > 0$, then
the number of iterations becomes
\begin{align*}
T = O\Big(
& \frac{\beta R^2}{m}\left(\frac{1}{\delta} + \log \frac{f(\xx^0) - f(\xx^*)}{\eps}\right)\\
& + \sqrt{\beta} R \log \frac{f(\xx^0) - f(\xx^*)}{\eps}
\Big)
\end{align*}
\label{thm:dense_l2}
\end{theorem}
Theorem~\ref{thm:dense_l2} is proved in Appendix~\ref{sec:proof_thm_dense_l2}.

\section{Numerical Example}

In order to numerically validate our algorithm, we run logistic regression on the 
UCI adult binary
classification dataset. In order to simulate a separable dataset, we first run gradient descent
on the whole data, and then discard the misclassified data points. This gives us a separable dataset.
Then, we run two variants of gradient descent: One with constant step size
given by $\beta^{-1}$, and one with increasing step size given by 
$\eta_t = \beta^{-1} f(\xx^0) / f(\xx^t)$, with no other change. This is motivated by our findings,
which suggest that the step size should increase proportionally to the decrease of the loss. As we
can see in Figure~\ref{fig:adult}, the error in the case of fixed step size decays as 
$\mathrm{poly}(1/t)$, while
in the case of increasing step size we have linear convergence (albeit with a low rate because the 
margins are in the order of $10^{-6}$).
\begin{figure}[htp]
\begin{center}
\includegraphics[width=\columnwidth]{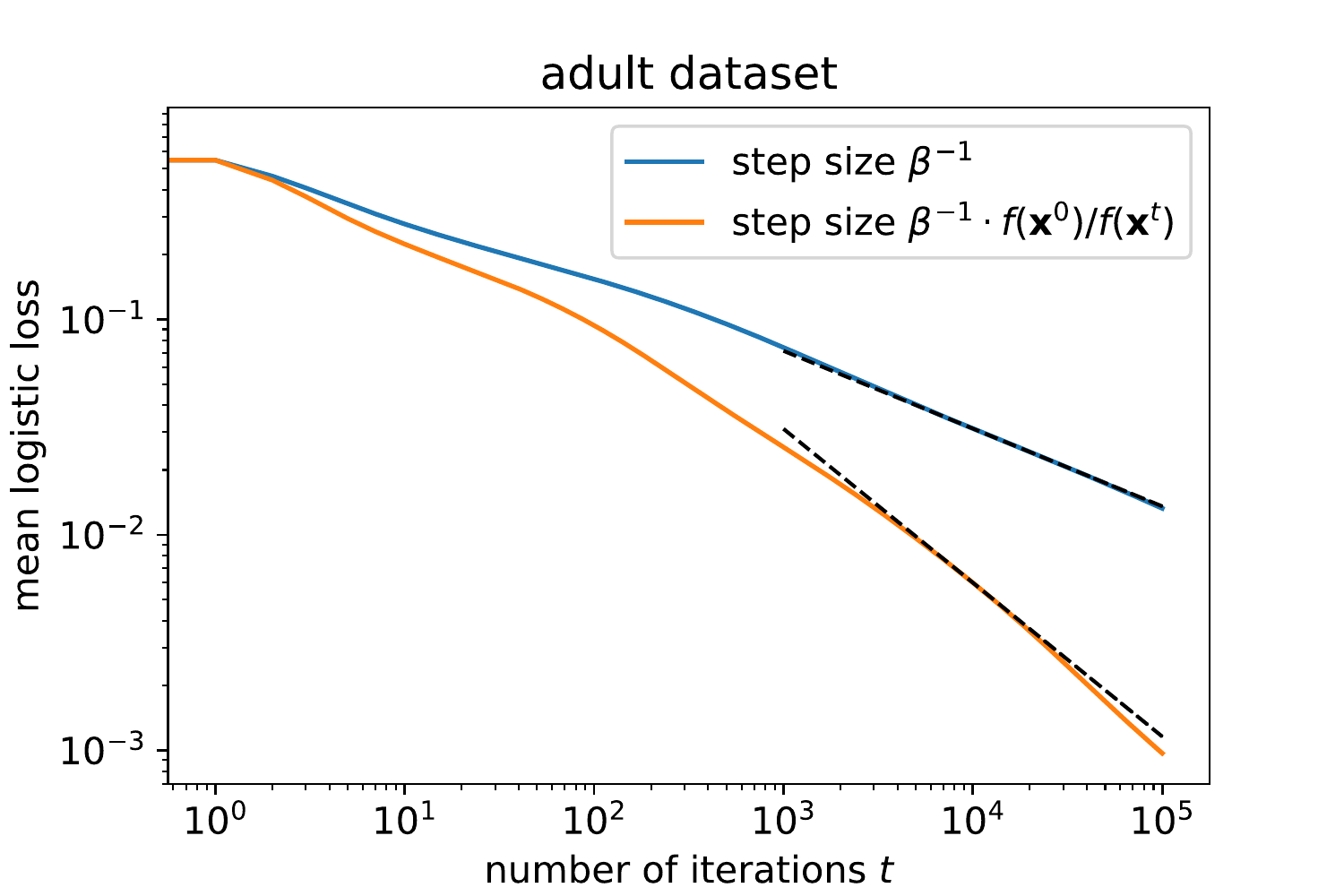}
\end{center}
\caption{Comparison of fixed vs increasing step size on logistic regression on adult dataset}
\label{fig:adult}
\end{figure}

\section{Acknowledgments}
We would like to thank the anonymous reviewers for valuable feedback, and pointing out 
a simplification to the choice of step size in Algorithm~\ref{alg:sparse}.

\FloatBarrier

\bibliography{main}
\bibliographystyle{icml2023}

\newpage
\appendix
\onecolumn

\section{Missing Proofs from Section~\ref{sec:sparse}}

\subsection{Main lemma on coordinate updates}

\begin{lemma}[Gradient lower bound]
Let $f:\mathbb{R}^n\rightarrow \mathbb{R}$ be a differentiable convex function and let
$\xx\in[-B',B']^n$, $\xx^*\in[-B,B]^n$ 
be two solutions %
for some parameters $B' \geq B > 0$.
For all $i\in[n]$ we define 
\[ \zeta_i = \begin{cases}
	\lambda & \text{ if $x_i = 0$}\\
	0 & \text{ if $|x_i| \geq B$ and $\nabla_i f(\xx) \cdot x_i < 0$}\\
	1 & \text{ otherwise}
\end{cases} \]
where $0 < \lambda \leq 1$, 
and let $i^* = \mathrm{argmax}_i \left\{\zeta_i \left|\nabla_i f(\xx)\right|\right\}$. 
Then, at least one of the following is true:
\begin{align*}
& \bullet \,\,\,\,\,\,\, \left|\nabla_{i^*} f(\xx)\right| 
\geq 
\frac{f(\xx) - f(\xx^*)}{\left\|\xx^*\right\|_1 + \lambda \left\|\xx\right\|_1}\\
&  \bullet \,\,\,\,\,\,\,\left|\nabla_{i^*} f(\xx)\right| \geq \frac{f(\xx) - f(\xx^*)}{\lambda^{-1}\left\|\xx^*\right\|_1 + \left\|\xx\right\|_1}  \text{ and } x_{i^*}\neq 0 \,.
\end{align*}
\label{lem:grad_lb}
\end{lemma}
\begin{proof}
Let $S = \{i\ |\ x_i \neq 0\}$ and $F = \{i\ |\ |x_i|<B \text{ or } \nabla_i f(\xx) \cdot x_i \geq 0\}$.
By convexity of $f$, we have
\begin{align*}
f(\xx^*) 
& \geq f(\xx) + \langle \nabla f(\xx), \xx^* - \xx\rangle \\
& \geq f(\xx) + \langle \nabla_F f(\xx), \xx^* - \xx\rangle \\
& = f(\xx) + \langle \nabla_F f(\xx), \xx^*\rangle - \langle \nabla_{S\cap F} f(\xx), \xx\rangle \\
& \geq f(\xx) - \left\|\nabla_F f(\xx)\right\|_\infty \left\|\xx^*\right\|_1 
- \left\|\nabla_{S\cap F} f(\xx)\right\|_\infty \left\|\xx\right\|_1\,,
\end{align*}
where the first inequality holds because for any $i\in[n]\backslash F$, by the fact that
$\left|x_i^*\right| \leq B$ and the definition of $F$,
\[ 
\nabla_i f(\xx) \left(\xx^* - \xx\right)_i
\geq \left|\nabla_i f(\xx)\right| \left(-B + B\right) = 0\,.
\]

Therefore
\begin{align}
\left\|\nabla_F f(\xx)\right\|_\infty \left\|\xx^*\right\|_1 
+ \left\|\nabla_{S\cap F} f(\xx)\right\|_\infty \left\|\xx\right\|_1 \geq 
f(\xx) - f(\xx^*)\,.
\label{eq:grad_lb_1}
\end{align}
Now, if $i^*\notin S$, which also implies $x_{i^*} = 0$, 
by definition of the $\zeta_i$'s and $i^*$ we have
\begin{align*}
\lambda \left\|\nabla_F f(\xx)\right\|_\infty 
= \lambda \left|\nabla_{i^*} f(\xx)\right| 
\geq \left\|\nabla_{S\cap F} f(\xx)\right\|_\infty\,.
\end{align*}
and so (\ref{eq:grad_lb_1}) implies
\begin{align*}
& \left|\nabla_{i^*} f(\xx)\right| \left\|\xx^*\right\|_1 
+ \lambda \left|\nabla_{i^*} f(\xx)\right| \left\|\xx\right\|_1 \geq f(\xx) - f(\xx^*)\\
\Rightarrow & \left|\nabla_{i^*} f(\xx)\right| \geq \frac{f(\xx) - f(\xx^*)}{\left\|\xx^*\right\|_1 + \lambda \left\|\xx\right\|_1}\,.
\end{align*}
Otherwise if $i^*\in S$, we have
\begin{align*}
\left\|\nabla_{S\cap F} f(\xx)\right\|_\infty 
= \left|\nabla_{i^*} f(\xx)\right| 
\geq \lambda \left\|\nabla_F f(\xx)\right\|_\infty\,,
\end{align*}
and so (\ref{eq:grad_lb_1}) implies 
\begin{align*}
& \lambda^{-1} \left|\nabla_{i^*} f(\xx)\right| \left\|\xx^*\right\|_1 + \left|\nabla_{i^*} f(\xx)\right| \left\|\xx\right\|_1 \geq f(\xx) - f(\xx^*)\\
\Rightarrow & \left|\nabla_{i^*} f(\xx)\right| \geq \frac{f(\xx) - f(\xx^*)}{\lambda^{-1} \left\|\xx^*\right\|_1 + \left\|\xx\right\|_1}\,.
\end{align*}
\end{proof}

\begin{lemma}[Coordinate update]
Let $f:\mathbb{R}^n\rightarrow \mathbb{R}_{> 0}$ be a twice continuously differentiable convex function that is $2\gamma$-second order robust and $\gamma^2$-multiplicatively smooth
with respect to the $\ell_1$ norm, for some $\gamma > 0$.
Let $\xx\in[-B',B']^n$ be a suboptimal solution such that 
$f(\xx) \geq f(\xx^*)$, 
where $\xx^*\in[-B,B]^n$ is some unknown solution
with $\gamma \left\|\xx^*\right\|_1 \geq 1$,
and $B' \geq B > 0$ are some parameters.
We make the update
\begin{align*}
\xx' = \xx - \eta \nabla_i f(\xx) \onev_i\,,
\end{align*}
where
$i$ is picked as in Lemma~\ref{lem:grad_lb} for some parameter $\lambda \in(0,1)$ and
$\eta = 0.5 \min\left\{\frac{1}{\gamma^2 f(\xx)}, \frac{1}{\gamma |\nabla_i f(\xx)|}\right\}$
is a step size.
Then, at least one of the following is true about the progress in decreasing $f$:
\begin{align*}
&\bullet \,\,\,\,\,\,\, f(\xx) - f(\xx') \geq \frac{\left(f(\xx) - f(\xx^*)\right)^2}{4 \gamma^2 f(\xx) \left(\left\|\xx^*\right\|_1 + \lambda\left\|\xx\right\|_1\right)^2} \\
&\bullet \,\,\,\,\,\,\, f(\xx) - f(\xx') \geq \frac{\left(f(\xx) - f(\xx^*)\right)^2}{4 \gamma^2 f(\xx) \left(\lambda^{-1} \left\|\xx^*\right\|_1 + \left\|\xx\right\|_1\right)^2} \text{ and } x_i \neq 0 \,,
\end{align*}
and the norm of the new solution is bounded as
$\left\|\xx'\right\|_\infty \leq \max\, \{B', B + \frac{1}{2\gamma}\}$.
In the case that $f(\xx) < f(\xx^*)$ we have $f(\xx') \leq f(\xx)$.
\label{lem:coord_upd}
\end{lemma}
\begin{proof}
We first consider a generic update $\xx' = \xx + \txx$. By Taylor's theorem
and the fact that $f$ is twice continuously differentiable, we have
\begin{align*}
f(\xx') = f(\xx) + \langle \nabla f(\xx), \txx \rangle + \frac{1}{2} \langle \txx, \nabla^2 f(\oxx) \txx\rangle \,,%
\end{align*}
for some $\oxx$ that is entrywise between $\xx$ and $\xx'$.

Since $f$ is $2\gamma$-second-order-robust and $\gamma^2$-multiplicatively-smooth with respect to the $\ell_1$ norm,
as long as the update is bounded in $\ell_1$ norm as 
\begin{align}
\left\|\txx\right\|_1 \leq 1/(2\gamma)
\label{eq:l1_neigh}
\end{align}
we have
\begin{align*}
f(\xx') 
& \leq f(\xx) + \langle \nabla f(\xx), \txx \rangle + \langle \txx, \nabla^2 f(\xx) \txx\rangle \\
& \leq f(\xx) + \langle \nabla f(\xx), \txx \rangle + \gamma^2 f(\xx) \left\|\txx\right\|_1^2\,.
\end{align*}
Note that the right hand side is minimized for 
\[ \txx = -\frac{H_1\left(\nabla f(\xx)\right)}{2\gamma^2 f(\xx)}\,, \] 
where $H_1$ is the hard thresholding operator that zeroes out all but the top entry in absolute value. 
This is a coordinate descent step. 
Our step will be slightly more careful so that it doesn't unnecessarily increase the 
sparsity of $\xx$. We consider the following coordinate step
\begin{align*}
\txx = -\eta \nabla_i f(\xx) \onev_i\,,
\end{align*}
where $\eta > 0$ and $i$ are as defined in the lemma statement.
We now have a function value decrease of 
\begin{align*}
f(\xx) - f(\xx')
& \geq \left(\eta - \eta^2 \gamma^2 f(\xx)\right) (\nabla_i f(\xx))^2\,.
\end{align*}
The term $\eta- \eta^2 \gamma^2 f(\xx)$ is maximized at $\eta = \frac{1}{2 \gamma^2 f(\xx)}$. In addition, to stay in the $\ell_1$ neighborhood where the Hessian is stable,
we need to satisfy (\ref{eq:l1_neigh}) by making sure that $\eta \leq \frac{1}{2\gamma \left|\nabla_i f(\xx)\right|}$. Based on these requirements, we pick
\begin{align*}
\eta = \min\left\{\frac{1}{2\gamma^2 f(\xx)}, \frac{1}{2\gamma |\nabla_i f(\xx)|}\right\}
\end{align*}
and conclude that
\begin{align*}
f(\xx) - f(\xx')
& \geq \min\left\{\frac{1}{4\gamma^2 f(\xx)}, \frac{1}{4\gamma |\nabla_i f(\xx)|}\right\} \left(\nabla_i f(\xx)\right)^2\\
& = \min\left\{\frac{(\nabla_i f(\xx))^2}{4\gamma^2 f(\xx)}, \frac{|\nabla_i f(\xx)|}{4\gamma}\right\}\,.
\end{align*}
Note that this is always $\geq 0$ and so we have $f(\xx') \leq f(\xx)$ even if $f(\xx) < f(\xx^*)$.
We now take two cases and use the two bullets of Lemma~\ref{lem:grad_lb} accordingly.
\paragraph{Case 1: $x_i = 0$.} The first bullet of Lemma~\ref{lem:grad_lb} has to be true, i.e.
\[ \left|\nabla_{i} f(\xx)\right| \geq \frac{f(\xx) - f(\xx^*)}{\left\|\xx^*\right\|_1 + \lambda\left\|\xx\right\|_1} \,.\]
Therefore,
\begin{align*}
f(\xx) - f(\xx')
& \geq \min\left\{\frac{(f(\xx) - f(\xx^*))^2}{4\gamma^2 f(\xx) \left(\left\|\xx^*\right\|_1 + \lambda\left\|\xx\right\|_1\right)^2}, 
\frac{f(\xx) - f(\xx^*)}{4\gamma\left(\left\|\xx^*\right\|_1 + \lambda\left\|\xx\right\|_1\right)}\right\}\\
& = \frac{(f(\xx) - f(\xx^*))^2}{4\gamma^2 f(\xx) \left(\left\|\xx^*\right\|_1 + \lambda\left\|\xx\right\|_1\right)^2}\,,
\end{align*}
where we used the facts that $f(\xx) - f(\xx^*) \leq f(\xx)$
and $\gamma \left\|\xx^*\right\|_1 \geq 1$.
\paragraph{Case 2: $x_i\neq 0$.} If the first bullet of Lemma~\ref{lem:grad_lb} is true, 
we can proceed as in the previous case. Otherwise, we use the second bullet of Lemma~\ref{lem:grad_lb}
and similarly get
\begin{align*}
f(\xx) - f(\xx')
& \geq \frac{(f(\xx) - f(\xx^*))^2}{4 \gamma^2 f(\xx) \left(\lambda^{-1} \left\|\xx^*\right\|_1 + \left\|\xx\right\|_1\right)^2}\,.
\end{align*}
Finally, in order to bound $\left\|\xx'\right\|_\infty$, we first note that 
$\left\|\xx\right\|_\infty \leq B'$. Now, by our choice of $i$ we have that 
either $\left|x_i\right| < B$, or $\nabla_i f(\xx) \cdot x_i > 0$. In the first case, we have
\begin{align*}
|x_i'| \leq |x_i| + |\tx_i| < B + \frac{1}{2\gamma}\,,
\end{align*}
where we used (\ref{eq:l1_neigh}). Otherwise, we have that $|x_i| \geq B$ and $\nabla_i f(\xx) \cdot x_i > 0$. This implies that $x_i$ and $\tx_i$ have different signs, so
\[ |x_i'| = |x_i + \tx_i| \leq \max\left\{|x_i|, |\tx_i|\right\} \leq \max\left\{B', \frac{1}{2\gamma}\right\}\,. \]
Therefore, in any case we have $|x_i'| \leq \max\left\{B', B + \frac{1}{2\gamma}\right\}$.
\end{proof}

\subsection{Theorems}

\subsubsection{Proof of Theorem~\ref{thm:sparse}}
\label{sec:proof_thm_sparse}

\begin{proof}
We will apply Lemma~\ref{lem:coord_upd} for $T$ iterations to obtain solutions $\xx^0,\dots,\xx^T$,
for some $T$ that will be defined later. The step size parameter $\lambda_t < 1$
disincentivizes updating zero entries of the solution vector. 
The logistic function 
$f$
is $2M$-second order robust and $M^2$-multiplicatively smooth with respect to the 
$\ell_1$ norm (Lemmas~\ref{lem:second_order_robustness} and \ref{lem:multiplicative_smoothness}), so 
Lemma~\ref{lem:coord_upd} can be applied with $\gamma = M$ and $B' = B + \frac{1}{2M}$.

Also, we can simplify the step size
used in Lemma~\ref{lem:coord_upd} to 
$\frac{1}{2M^2 f(\xx)}$.
This is because 
\begin{align*}
\left\|\nabla f(\xx)\right\|_\infty
& = \left\|\AA^\top (1 - \sigma(\AA\xx))\right\|_\infty\\
& \leq M \left\|1 - \sigma(\AA\xx)\right\|_1 \\
& \leq M f(\xx)\,,
\end{align*}
where we used the fact that $1-\sigma(t) = 1/(1+e^t) \leq \log (1+e^{-t})$. The
inequality holds because
the function $g(t) = (1+e^t) \log (1+e^{-t})$ is decreasing
($g'(t) = e^t\log(1+e^{-t}) - 1 \leq 0$) and goes to $1$ as $t\rightarrow\infty$.

The progress bound of Lemma~\ref{lem:coord_upd} depends on $\left\|\xx^t\right\|_1$. 
Based on the guarantee of Lemma~\ref{lem:coord_upd} about $\left\|\xx^t\right\|_\infty$, 
we can derive the following bound:
\begin{align*}
\left\|\xx^t\right\|_1
& \leq \left\|\xx^t\right\|_0 \left\|\xx^t\right\|_\infty\\
& \leq 
	\left\|\xx^t\right\|_0\left(B + \frac{1}{2M}\right)\\
& \leq \left\|\xx^t\right\|_0 (3/2)B\,.
\end{align*}

We can now bound the sparsity. 
Note that the sparsity increases by at most $1$ every time the first bullet of 
Lemma~\ref{lem:coord_upd} is true, and does not increase when the second bullet is true.
Therefore, the progress in each sparsity-increasing iteration, based on our choice of
$\lambda_t$, is
\begin{align*}
f(\xx^t) - f(\xx^{t+1}) 
& \geq \frac{\left(f(\xx^t)-f(\xx^*)\right)^2}{4f(\xx^t)M^2 \left(\left\|\xx^*\right\|_1 + \lambda_t\left\|\xx^t\right\|_1\right)^2}\\
& \geq \frac{\left(f(\xx^t)-f(\xx^*)\right)^2}{4f(\xx^t)M^2 \left(\left\|\xx^*\right\|_1 + B_1\right)^2}\\
& \geq \frac{\left(f(\xx^t)-f(\xx^*)\right)^2}{16f(\xx^t)M^2\left\|\xx^*\right\|_1^2}\,.
\end{align*}
The following auxiliary lemma will help us turn this into a convergence result:
\begin{lemma}
Consider a function $f:\mathbb{R}^n\rightarrow\mathbb{R}_{>0}$ 
and a sequence $\xx^0,\xx^1,\dots$ of iterates such that for all $t$
\begin{align}
f(\xx^t) - f(\xx^{t+1}) \geq \alpha \frac{(f(\xx^t) - f(\xx^*))^2}{f(\xx^t)}
\label{eq:aux_progress}
\end{align}
for some $\xx^*$ with $f(\xx^*) \leq \min_t f(\xx^t)$. Then,
\begin{align*}
f(\xx^T)\leq (1+\delta) f(\xx^*) + \eps
\end{align*}
after 
\[ T 
\leq 2 \alpha^{-1} \left(\frac{1}{\delta} + \log \frac{f(\xx^0) - f(\xx^*)}{\eps}\right) \]
iterations.
\label{lem:aux_convergence}
\end{lemma}
\begin{proof}
Let $\ot$ be the smallest $t \geq 0$ for which $f(\xx^{\ot}) \leq 2 f(\xx^*)$
or $f(\xx^{\ot}) \leq f(\xx^*) + \eps$, and let $\ot=\infty$
if this never happens. Therefore, for all $t < \ot$ we have 
$f(\xx^t) \geq 2 f(\xx^*)
\Rightarrow 
\frac{f(\xx^t) -f(\xx^*)}{f(\xx^t)} \geq \frac{1}{2}$,
Then, (\ref{eq:aux_progress}) implies
\begin{align*}
&
 f(\xx^t) - f(\xx^{t+1}) \geq \frac{\alpha}{2} \left(f(\xx^t) - f(\xx^*)\right)\\
\Rightarrow &
 f(\xx^{t+1}) - f(\xx^*) \leq \left(1-\frac{\alpha}{2}\right) \left(f(\xx^t) - f(\xx^*)\right)\,.
\end{align*}
Applying these for all $t$, we get
\begin{align*}
 f(\xx^{\ot}) - f(\xx^*) \leq \left(1-\frac{\alpha}{2}\right)^{\ot} \left(f(\xx^0) - f(\xx^*)\right)\,,
\end{align*}
which directly implies that
\begin{align*}
\ot \leq 2 \alpha^{-1} \log \frac{f(\xx^0) - f(\xx^*)}{\eps}\,.
\end{align*}
Now, let $t\geq \ot$. If $f(\xx^{\ot}) \leq f(\xx^*) + \eps$ we are done and there is nothing to prove. Otherwise, $f(\xx^t)\le  f(\xx^{\ot}) \leq 2 f(\xx^*)$
 and therefore
(\ref{eq:aux_progress}) implies
\begin{align*}
f(\xx^t) - f(\xx^{t+1}) \geq \frac{\alpha}{2f(\xx^*)} (f(\xx^t) - f(\xx^*))^2\,.
\end{align*}
This is a well known recurrence that leads to the bound
\begin{align*}
f(\xx^T) 
\leq f(\xx^*) + \frac{2f(\xx^*)}{\alpha (T-\ot)}
= \left(1 + \frac{2}{\alpha (T-\ot)} \right) f(\xx^*)\,,
\end{align*}
therefore 
\[ T \leq 
\frac{2}{\alpha \delta} + \ot
\leq 2 \alpha^{-1} \left(\frac{1}{\delta} + \log \frac{f(\xx^0) - f(\xx^*)}{\eps}\right)\,. \]
\end{proof}
Lemma~\ref{lem:aux_convergence} implies that the total number of sparsity-increasing
iterations is
\begin{align*}
s' := \left\|\xx^T\right\|_0 \leq 32\left\|\xx^*\right\|_1^2 M^2\left(\frac{1}{\delta} + \log \frac{f(\xx^0) - f(\xx^*)}{\eps}\right)\,.
\end{align*}

Now it remains to bound the total number of iterations in which the second bullet of 
Lemma~\ref{lem:coord_upd} is true. These iterations do not increase the sparsity.
We have 
\begin{align*}
f(\xx^t) - f(\xx^{t+1}) 
& \geq \frac{(f(\xx^t) - f(\xx^*))^2}{4 f(\xx^t) M^2\left(\lambda_t^{-1} \left\|\xx^*\right\|_1 + \left\|\xx^t\right\|_1\right)^2}\,.
\end{align*}
Note that
\begin{align*}
\lambda_t^{-1} \left\|\xx^*\right\|_1 
& = \max\left\{\left\|\xx^t\right\|_1 \left\|\xx^*\right\|_1/ B_1,\left\|\xx^*\right\|_1\right\} \\
& \leq \max\left\{C \left\|\xx^t\right\|_1,\left\|\xx^*\right\|_1\right\} \,,
\end{align*}
If $C\left\|\xx^t\right\|_1 \geq \left\|\xx^*\right\|_1$, then 
\begin{align*}
f(\xx^t) - f(\xx^{t+1}) 
& \geq \frac{(f(\xx^t) - f(\xx^*))^2}{4 f(\xx^t) M^2(C+1)^2\left\|\xx^t\right\|_1^2}\\
& \geq \frac{(f(\xx^t) - f(\xx^*))^2}{9 f(\xx^t) M^2B^2(C+1)^2 \left\|\xx^t\right\|_0^2 }\\
& \geq \frac{(f(\xx^t) - f(\xx^*))^2}{9f(\xx^t) M^2B^2(C+1)^2 (s')^2 }\,.
\end{align*}
By Lemma~\ref{lem:aux_convergence}, there can only be
\begin{align*}
T 
& = 9 M^2 B^2 (C+1)^2 (s')^2 \left(\frac{1}{\delta} + \log \frac{f(\xx^0) - f(\xx^*)}{\eps}\right)\\
& = O\left((s')^2 M^2 B^2 C^2  \left(\frac{1}{\delta} + \log \frac{f(\xx^0) - f(\xx^*)}{\eps}\right)\right)
\end{align*}
iterations.
Similarly if $C\left\|\xx^t\right\|_1 < \left\|\xx^*\right\|_1$ we have at most
\begin{align*}
& O\left(\left((s')^2 + \left\|\xx^*\right\|_0^2\right) M^2 B^2 \left(\frac{1}{\delta} + \log \frac{f(\xx^0) - f(\xx^*)}{\eps}\right)\right)
\end{align*}
such iterations, so the result follows.

\ifx 0
-------------

Based on the guarantee of Lemma~\ref{lem:coord_upd}, we get the following bound on the 
$\ell_1$ norm of $\xx^t$ at all times:
\begin{align*}
\left\|\xx^t\right\|_1
\leq n \left\|\xx^t\right\|_\infty
\leq 
	n\left(B + \frac{1}{2M}\right)
\leq (3/2) nB\,.
\end{align*}
Let $\ot$ be the smallest $t \geq 0$ for which $f(\xx^{\ot}) \leq 2 f(\xx^*)$
or $f(\xx^{\ot}) \leq f(\xx^*) + \eps$, and let $\ot=\infty$
if this never happens. Therefore, for all $t < \ot$ we have 
$f(\xx^t) \geq 2 f(\xx^*)
\Rightarrow 
\frac{f(\xx^t) -f(\xx^*)}{f(\xx^t)} \geq \frac{1}{2}$,
and so the statement of Lemma~\ref{lem:coord_upd} gives:
\begin{align*}
f(\xx^{t}) - f(\xx^{t+1}) 
& \geq \frac{f(\xx^t) - f(\xx^*)}{8 M^2 (\left\|\xx^*\right\|_1 + \left\|\xx^t\right\|_1)^2}\\
& \geq \frac{f(\xx^t) - f(\xx^*)}{8 n^2 M^2 (B + (3/2) B)^2}\\
& \geq \frac{f(\xx^t) - f(\xx^*)}{50 n^2 M^2B^2}\,,
\end{align*}
where we used the fact that $\left\|\xx^*\right\|_1 \leq n \left\|\xx^*\right\|_\infty \leq nB$.

\[  \left\|\xx^*\right\|_1 \leq s B \]

Compared to Corollary~\ref{cor:dense}, we have the differences that
and that we have the following tighter bounds because of sparsity:
\begin{align*}
& \left\|\xx^*\right\|_1 \leq s B \\
& \left\|\xx^t\right\|_1 \leq 
\left\|\xx^t\right\|_0 \left\|\xx^t\right\|_\infty \leq  
\left\|\xx^t\right\|_0 (3/2) B
\,.
\end{align*}

Equivalently,
\begin{align*}
f(\xx^{t+1}) - f(\xx^{*}) 
\leq \left(1 - \frac{1}{50 n^2 M^2 B^2}\right)
(f(\xx^{t}) - f(\xx^{*}))\,,
\end{align*}
and summing up these for $t\in\{0,1,\dots,\ot-1\}$, we get
\begin{align*}
f(\xx^{\ot}) - f(\xx^{*}) 
& \leq \left(1 - \frac{1}{50 n^2M^2B^2}\right)^{\ot}
(f(\xx^{0}) - f(\xx^{*}))\\
& \leq \eps\,,
\end{align*}
as long as 
\[ \ot \geq 50 n^2M^2B^2\log \frac{f(\xx^0) - f(\xx^*)}{\eps}\,,\]
therefore we conclude that $\ot$ is at most this quantity.

Now we consider the iterations $t \geq \ot$. If $f(\xx^{\ot}) \leq f(\xx^*) + \eps$
there are no such iterations and we are done.
Therefore we have that $f(\xx^{\ot}) \leq 2 f(\xx^*)$.
We again use Lemma~\ref{lem:coord_upd} for all such $t$, which gives 
\begin{align*}
f(\xx^{t}) - f(\xx^{t+1}) 
& \geq \frac{(f(\xx^t) - f(\xx^*))^2}{4 M^2 f(\xx^t) (\left\|\xx^*\right\|_1 + \left\|\xx^t\right\|_1)^2}\\
& \geq \frac{(f(\xx^t) - f(\xx^*))^2}{25 f(\xx^t) n^2M^2B^2}\\
& \geq \frac{(f(\xx^t) - f(\xx^*))^2}{50 f(\xx^*) n^2M^2B^2}\,.
\end{align*}
By known convergence results, this recurrence leads to the bound
\begin{align*}
f(\xx^{T})
& \leq f(\xx^*) + \frac{100 f(\xx^*) n^2 M^2B^2}{T-\ot}\\
& \leq f(\xx^*)\left(1 + \frac{100 n^2 M^2B^2}{T-\ot}\right)\,,
\end{align*}
implying that $f(\xx^T) \leq (1+\delta) f(\xx^*)$ after 
\begin{align*}
& T-\ot = O\left(n^2 M^2B^2 \frac{1}{\delta}\right)
\end{align*}
additional iterations after $\ot$. Therefore, the total number of iterations 
to achieve $f(\xx^T) \leq (1+\delta) \cdot f(\xx^*) + \eps$ is
\begin{align*}
& O\left(n^2M^2B^2 \left(\frac{1}{\delta} + \log \frac{f(\xx^0) - f(\xx^*)}{\eps}\right)\right)\,.
\end{align*}

We have 
\begin{align*}
& \max\left\{ \left\|\xx^*\right\|_1 + \lambda_t\left\|\xx^t\right\|_1,
\lambda_t^{-1}\left\|\xx^*\right\|_1 + \left\|\xx^t\right\|_1 \right\}\\
& \leq 
 \max\left\{ \left\|\xx^*\right\|_1 + \left\|\xx^t\right\|_1,
c^{-1} \left\|\xx^t\right\|_1 + \left\|\xx^t\right\|_1 \right\}\\
& \leq 
 \left\|\xx^*\right\|_1 +
(1 + c^{-1}) \left\|\xx^t\right\|_1\\
& \leq 
 \left\|\xx^*\right\|_1 +
\frac{3}{2} (1+c^{-1})\left\|\xx^t\right\|_0 B\\
& \leq 
 \left\|\xx^*\right\|_1 +
\frac{3}{2} (1+c^{-1})\left\|\xx^T\right\|_0 B\,.
\end{align*}
As a result, the progress bound of Lemma~\ref{lem:coord_upd} becomes 
\begin{align*}
f(\xx^t) - f(\xx^{t+1})
\geq \frac{\left(f(\xx) - f(\xx^*)\right)^2}{4f(\xx) M^2 \left(
 \left\|\xx^*\right\|_1 +
\frac{3}{2} (1+c^{-1})\left\|\xx^T\right\|_0 B
\right)^2}\,,
\end{align*}
and, analogously to the proof of Corollary~\ref{cor:dense} and using the fact that
$\left\|\xx^*\right\|_1 \leq \left\|\xx^*\right\|_0B$,
the 
total number of iterations is bounded by
\begin{align*}
T = O\left(
(1+c^{-1})^2
\left(
\left\|\xx^T\right\|_0^2 
+ \left\|\xx^*\right\|_0^2 
\right)
M^2 B^2
\left(\frac{1}{\delta} + \log \frac{f(\xx^0) - f(\xx^*)}{\eps}\right)\right)\,.
\end{align*}
\fi
\end{proof}

\subsubsection{Proof of Theorem~\ref{thm:sparse2}}
\label{sec:proof_thm_sparse2}

\begin{proof}
We move similarly to the proof of Theorem~\ref{thm:sparse}, but now we can strengthen
Lemma~\ref{lem:coord_upd} because $\xx^t$ is fully corrected for all $t$, i.e. 
$\nabla_i f(\xx^t) = 0$ for all $i\in\mathrm{supp}(\xx^t)$.
As in the proof of Lemma~\ref{lem:coord_upd}, we can lower bound the amount of progress
as a function of $\left\|\nabla f(\xx^t)\right\|_\infty$ as follows:
\begin{align*}
f(\xx^t) - f(\xx^{t+1}) \geq
	\frac{(\nabla_i f(\xx^t))^2}{4M^2 f(\xx^t)}
\,.
\end{align*}
Now, by convexity of $f$ we have
\begin{align}
\langle \nabla f(\xx^t), \xx^t - \xx^*\rangle \geq f(\xx^t) - f(\xx^*)\,.
\label{eq:s3grad}
\end{align}
Because of fully corrective steps we have $\langle \nabla f(\xx^t), \xx^t\rangle = 0$, and so the 
left hand side of (\ref{eq:s3grad}) is upper bounded by 
$\left\|\nabla f(\xx^t)\right\|_\infty \left\|\xx^*\right\|_1$. As a result, we have
\begin{align*}
\left\|\nabla f(\xx^t)\right\|_\infty^2 \geq \frac{\left(f(\xx) - f(\xx^*)\right)^2}{\left\|\xx^*\right\|_1^2}\,,
\end{align*}
and so we get the progress bound of 
\begin{align*}
f(\xx^t) - f(\xx^{t+1}) 
& \geq 
	\frac{(f(\xx^t)-f(\xx^*))^2}{4M^2 f(\xx^t) \left\|\xx^*\right\|_1^2}
\,.
\end{align*}
By Lemma~\ref{lem:aux_convergence} (similarly to
the proof of Theorem~\ref{thm:sparse}), this progress bound leads to
a sparsity of
\begin{align*}
s' := \left\|\xx^T\right\|_0 \leq O\left(\left\|\xx^*\right\|_1^2 M^2 \left(\frac{1}{\delta} + 
\log \frac{f(\xx^0) - f(\xx^*)}{\eps}\right)\right)
\end{align*}
and the same number of iterations.
\end{proof}

\section{Missing Proofs from Section~\ref{sec:dense}}
\subsubsection{Proof of Corollary~\ref{cor:dense}}
\label{sec:proof_thm_dense}

\begin{proof}

We will apply Lemma~\ref{lem:coord_upd} for $T$ iterations to obtain solutions $\xx^0,\dots,\xx^T$,
where for some $T$ that will be defined later.  
The logistic function 
$f$
is $2M$-second order robust and $M^2$-multiplicatively smooth with respect to the $\ell_1$ norm, so 
Lemma~\ref{lem:coord_upd} can be applied with $\gamma = M$ and $B' = B + \frac{1}{2M}$.

\ifx 0 
Now, notice that Lemma~\ref{lem:coord_upd} guarantees that $f(\xx^t)$ is non-increasing.
Therefore for any $t$ we have $f(\xx^t) \leq f(\zerov) \leq m$, implying
\begin{align*}
e^{M\left\|\xx^t\right\|_\infty} \leq 
	\sum\limits_{i=1}^n (e^{Mx_i^t} + e^{-Mx_i^t})
	\leq 
	\frac{4nme^{MB_\infty}}{\eps}\,, 
\end{align*}
and so we can get the following bound on the $\ell_1$ norm of $\xx^t$ at all times:
\fi
We get the following bound on the 
$\ell_1$ norm of $\xx^t$ at all times:
\begin{align*}
\left\|\xx^t\right\|_1
\leq n \left\|\xx^t\right\|_\infty
\leq 
	n\left(B + \frac{1}{2M}\right)
\leq (3/2) nB\,.
\end{align*}
Let $\ot$ be the smallest $t \geq 0$ for which $f(\xx^{\ot}) \leq 2 f(\xx^*)$
or $f(\xx^{\ot}) \leq f(\xx^*) + \eps$, and let $\ot=\infty$
if this never happens. Therefore, for all $t < \ot$ we have 
$f(\xx^t) \geq 2 f(\xx^*)
\Rightarrow 
\frac{f(\xx^t) -f(\xx^*)}{f(\xx^t)} \geq \frac{1}{2}$,
and so the statement of Lemma~\ref{lem:coord_upd} gives:
\begin{align*}
f(\xx^{t}) - f(\xx^{t+1}) 
& \geq \frac{f(\xx^t) - f(\xx^*)}{8 M^2 (\left\|\xx^*\right\|_1 + \left\|\xx^t\right\|_1)^2}\\
& \geq \frac{f(\xx^t) - f(\xx^*)}{8 n^2 M^2 (B + (3/2) B)^2}\\
& \geq \frac{f(\xx^t) - f(\xx^*)}{50 n^2 M^2B^2}\,,
\end{align*}
where we used the fact that $\left\|\xx^*\right\|_1 \leq n \left\|\xx^*\right\|_\infty \leq nB$.
Equivalently,
\begin{align*}
f(\xx^{t+1}) - f(\xx^{*}) 
\leq \left(1 - \frac{1}{50 n^2 M^2 B^2}\right)
(f(\xx^{t}) - f(\xx^{*}))\,,
\end{align*}
and summing up these for $t\in\{0,1,\dots,\ot-1\}$, we get
\begin{align*}
f(\xx^{\ot}) - f(\xx^{*}) 
& \leq \left(1 - \frac{1}{50 n^2M^2B^2}\right)^{\ot}
(f(\xx^{0}) - f(\xx^{*}))\\
& \leq \eps\,,
\end{align*}
as long as 
\[ \ot \geq 50 n^2M^2B^2\log \frac{f(\xx^0) - f(\xx^*)}{\eps}\,,\]
therefore we conclude that $\ot$ is at most this quantity.

Now we consider the iterations $t \geq \ot$. If $f(\xx^{\ot}) \leq f(\xx^*) + \eps$
there are no such iterations and we are done.
Therefore we have that $f(\xx^{\ot}) \leq 2 f(\xx^*)$.
We again use Lemma~\ref{lem:coord_upd} for all such $t$, which gives 
\begin{align*}
f(\xx^{t}) - f(\xx^{t+1}) 
& \geq \frac{(f(\xx^t) - f(\xx^*))^2}{4 M^2 f(\xx^t) (\left\|\xx^*\right\|_1 + \left\|\xx^t\right\|_1)^2}\\
& \geq \frac{(f(\xx^t) - f(\xx^*))^2}{25 f(\xx^t) n^2M^2B^2}\\
& \geq \frac{(f(\xx^t) - f(\xx^*))^2}{50 f(\xx^*) n^2M^2B^2}\,.
\end{align*}
By known convergence results, this recurrence leads to the bound
\begin{align*}
f(\xx^{T})
& \leq f(\xx^*) + \frac{100 f(\xx^*) n^2 M^2B^2}{T-\ot}\\
& = f(\xx^*)\left(1 + \frac{100 n^2 M^2B^2}{T-\ot}\right)\,,
\end{align*}
implying that $f(\xx^T) \leq (1+\delta) f(\xx^*)$ after 
\begin{align*}
& T-\ot = O\left(n^2 M^2B^2 \frac{1}{\delta}\right)
\end{align*}
additional iterations after $\ot$. Therefore, the total number of iterations 
to achieve $f(\xx^T) \leq (1+\delta) \cdot f(\xx^*) + \eps$ is
\begin{align*}
& O\left(n^2M^2B^2 \left(\frac{1}{\delta} + \log \frac{f(\xx^0) - f(\xx^*)}{\eps}\right)\right)\,.
\end{align*}
\end{proof}

\subsection{Gradient update lemma}
\begin{lemma}
Let $f:\mathbb{R}^n\rightarrow \mathbb{R}_{>0}$ be a twice continuously differentiable convex function that is $\gamma$-second order robust 
and $\mu$-multiplicatively smooth with respect to a norm $\left\|\cdot\right\|$ 
for some $\gamma, \mu > 0$.
Given a solution $\xx\in\mathbb{R}^n$, we make the update
\begin{align*}
\xx' = \xx - \eta \ggg\,,
\end{align*}
where 
\begin{align}
\ggg = \mathrm{argmin}_{\ggg \in\mathbb{R}^n}\, 
\langle \nabla f(\xx), -\ggg\rangle + \frac{1}{2} \left\|\ggg\right\|^2 
\label{eq:g_opt}
\end{align}
and
$\eta \leq \min\left\{\frac{1}{2\mu f(\xx)}, \frac{1}{\gamma \left\|\nabla f(\xx)\right\|}\right\}$
is a step size.
Then, the progress in decreasing $f$ is:
\begin{align*}
f(\xx) - f(\xx') \geq \frac{\eta}{2} \left\|\ggg\right\|^2\,.
\end{align*}
\label{lem:progress_lp_norm}
\end{lemma}
\begin{proof}
We first consider a generic update $\xx' = \xx + \txx$. By Taylor's theorem
and the fact that $f$ is twice continuously differentiable, we have
\begin{align*}
f(\xx') = f(\xx) + \langle \nabla f(\xx), \txx \rangle + \frac{1}{2} \langle \txx, \nabla^2 f(\oxx) \txx\rangle \,,%
\end{align*}
for some $\oxx$ that is entrywise between $\xx$ and $\xx'$.

Since $f$ is $\gamma$-second order robust 
and $\mu$-multiplicatively-smooth with respect to the norm $\left\|\cdot\right\|$,
as long as the update is bounded as 
\begin{align}
\left\|\txx\right\| \leq 1/\gamma\,,
\label{eq:grad_lp_neigh}
\end{align}
we have
\begin{align*}
f(\xx') 
& \leq f(\xx) + \langle \nabla f(\xx), \txx \rangle + \langle \txx, \nabla^2 f(\xx) \txx\rangle \\
& \leq f(\xx) + \langle \nabla f(\xx), \txx \rangle + \mu f(\xx) \left\|\txx\right\|^2\,.
\end{align*}
Note that the right hand side is minimized for 
\[ \txx = -\frac{1}{2\mu f(\xx)}\ggg\,. \] 
In addition, to stay in the neighborhood where the Hessian is stable,
we need to satisfy (\ref{eq:grad_lp_neigh}).
Based on these requirements, we make the update $\txx = -\eta \ggg$, where
\begin{align*}
\eta = \min\left\{\frac{1}{2\mu f(\xx)}, \frac{1}{\gamma\left\|\ggg\right\|}\right\}\,.
\end{align*}
Also, by the first-order optimality condition of (\ref{eq:g_opt}) it directly follows that
$\langle \nabla f(\xx), \ggg\rangle = \left\|\ggg\right\|^2$. Therefore,
\begin{align*}
f(\xx) - f(\xx')
& \geq \eta\langle \nabla f(\xx), \ggg\rangle - \eta^2\mu f(\xx) \left\|\ggg\right\|^2\\
& = \eta\left(1 - \eta\mu f(\xx)\right) \left\|\ggg\right\|^2\\
& \geq \frac{\eta}{2} \left\|\ggg\right\|^2\,,
\end{align*}
where the last inequality follows by our choice of step size.
\end{proof}
\begin{lemma}[Gradient update]
Let $f:\mathbb{R}^n\rightarrow \mathbb{R}_{>0}$ be a twice continuously differentiable convex function that is $\gamma$-second order robust 
and $\mu$-multiplicatively smooth with respect to the $\ell_2$ norm for some $\gamma, \mu > 0$.
Let $\xx\in\mathbb{R}^n$ be a solution such that $f(\xx) > f(\xx^*)$,
where $\xx^*\in\mathbb{R}^n$ is an unknown solution with
$\left\|\xx - \xx^*\right\|_2 \leq R$ 
for some $R>0$.
We make the update
\begin{align*}
\xx' = \xx - \eta \nabla f(\xx)\,,
\end{align*}
where
$\eta = \min\left\{\frac{1}{2\mu f(\xx)}, \frac{1}{\gamma \left\|\nabla f(\xx)\right\|_2}\right\}$
is a step size.
Then, the progress in decreasing $f$ is:
\begin{align*}
f(\xx) - f(\xx')
& \geq \min\left\{\frac{(f(\xx) - f(\xx^*))^2}{4\mu f(\xx) R^2}, 
\frac{f(\xx) - f(\xx^*)}{2\gamma R}\right\}\,.
\end{align*}
Additionally, as long as $\xx'$ is still suboptimal with respect to $\xx^*$, i.e.
$f(\xx') > f(\xx^*)$, the distance to $\xx^*$ decreases:
$\left\|\xx' - \xx^*\right\|_2 \leq \left\|\xx - \xx^*\right\|_2$.
Finally, if $f(\xx) \leq f(\xx^*)$, then $f(\xx') \leq f(\xx)$.
\label{lem:grad_upd}
\end{lemma}
\begin{proof}
We apply Lemma~\ref{lem:progress_lp_norm} with the $\ell_2$ norm, which gives $\ggg = \nabla f(\xx)$
and so
\begin{align*}
f(\xx) - f(\xx') 
& \geq \frac{\eta}{2} \left\|\nabla f(\xx)\right\|_2^2\\
& = \min\left\{\frac{1}{4\mu f(\xx)}, \frac{1}{2\gamma\left\|\nabla f(\xx)\right\|_2}\right\} \left\|\nabla f(\xx)\right\|_2^2\\
& = \min\left\{\frac{\left\|\nabla f(\xx)\right\|_2^2}{4\mu f(\xx)}, \frac{\left\|\nabla f(\xx)\right\|_2}{2\gamma}\right\}\,.
\end{align*}
This takes care of the case $f(\xx) \leq f(\xx^*)$, since it shows that $f(\xx') \leq f(\xx)$.
Now we deal with the case $f(\xx) > f(\xx^*)$. By convexity we have 
\begin{align*}
 f(\xx^*) 
 & \geq f(\xx) + \langle \nabla f(\xx), \xx^* - \xx \rangle \\
& \geq f(\xx) - \left\|\nabla f(\xx)\right\|_2 \left\|\xx^* - \xx\right\|_2\\
& \geq f(\xx) - \left\|\nabla f(\xx)\right\|_2 R
\,,
\end{align*}
which gives
\[ \left\|\nabla f(\xx)\right\|_2^2 
\geq \frac{(f(\xx) - f(\xx^*))^2}{R^2} \,,\]
and so 
\begin{align*}
f(\xx) - f(\xx')
& \geq \min\left\{\frac{(f(\xx) - f(\xx^*))^2}{4\mu f(\xx) R^2}, 
\frac{f(\xx) - f(\xx^*)}{2\gamma R}\right\}\,.
\end{align*}
For the norm bound, we suppose that $f(\xx') > f(\xx^*)$ (otherwise we are done). We have
\begin{align*}
& \left\|\xx' - \xx^*\right\|_2^2
- \left\|\xx - \xx^*\right\|_2^2 \\
& = \left\|\xx' - \xx\right\|_2^2 + 2 \langle \xx - \xx^*, \xx' - \xx\rangle\\
& = \eta^2 \left\|\nabla f(\xx)\right\|_2^2 -2 \eta\langle \xx - \xx^*, \nabla f(\xx)\rangle\,.
\end{align*}
Now, note that 
\[ \frac{\eta}{2} \left\|\nabla f(\xx)\right\|_2^2 \leq f(\xx) - f(\xx') \leq f(\xx) - f(\xx^*) \]
and by convexity $\langle \xx - \xx^*, \nabla f(\xx)\rangle \geq f(\xx) - f(\xx^*)$, so
\begin{align*}
& \left\|\xx' - \xx^*\right\|_2^2 - \left\|\xx - \xx^*\right\|_2^2 \\
& = \eta^2 \left\|\nabla f(\xx)\right\|_2^2 -2 \eta\langle \xx - \xx^*, \nabla f(\xx)\rangle\\
& \leq 0\,.
\end{align*}
\end{proof}

\subsection{Proof of Theorem~\ref{thm:dense_l2}}
\label{sec:proof_thm_dense_l2}
\begin{proof}
We repeatedly use Lemma~\ref{lem:grad_upd} to obtain iterates
$\xx^0,\xx^1,\dots,\xx^{T}$. Note that as long as $f(\xx^t) > f(\xx^*)$, we have 
$\left\|\xx^t - \xx^*\right\|_2 \leq \left\|\xx^0 - \xx^*\right\|_2 := R$.
We have 
\begin{align*}
f(\xx^t) - f(\xx^{t+1})
& \geq \min\left\{\frac{(f(\xx^t) - f(\xx^*))^2}{4\mu f(\xx^t) \left\|\xx^t - \xx^*\right\|_2^2}, 
\frac{f(\xx^t) - f(\xx^*)}{2\gamma\left\|\xx^t - \xx^*\right\|_2}\right\}\\
& \geq \min\left\{\frac{(f(\xx^t) - f(\xx^*))^2}{4\mu f(\xx^t) R^2}, 
\frac{f(\xx^t) - f(\xx^*)}{2\gamma R}\right\}\,.
\end{align*}
Let there be $T_1$ iterations where the first branch of the minimum is smaller, and $T_2$ where the 
second one is smaller, and suppose that $f(\xx^T) > (1+\delta) f(\xx^*) + \eps$, where $T = T_1+T_2$.
By Lemma~\ref{lem:aux_convergence} we immediately obtain that
\begin{align*}
T_1 \leq 8 \mu R^2 \left(\frac{1}{\delta} + \log \frac{f(\xx^0) - f(\xx^*)}{\eps}\right)\,.
\end{align*}
For $T_2$, a standard recurrence shows that 
\begin{align*}
T_2 \leq 4\gamma R \log \frac{f(\xx^0) - f(\xx^*)}{\eps}\,.
\end{align*}
Therefore the total number of iterations is bounded by 
\begin{align*}
T \leq O\left(\mu R^2 \left(\frac{1}{\delta} + \log \frac{f(\xx^0) - f(\xx^*)}{\eps}\right)
+ \gamma R \log \frac{f(\xx^0) - f(\xx^*)}{\eps}\right)\,.
\end{align*}
Replacing $\mu \leq \beta m^{-1}$ and $\gamma \leq 2\sqrt{\beta}$ we obtain the desired result.
\end{proof}

\end{document}